\documentclass[11pt, a4paper]{article}

\usepackage{amsmath}
\usepackage{amssymb}
\usepackage{mathtools}
\usepackage{colortbl}
\usepackage{graphicx}
\graphicspath{{plots/}}
\usepackage{hyperref}
\usepackage{geometry}
\usepackage{algorithm}
\usepackage{algpseudocode}
\usepackage{booktabs}
\usepackage{caption}
\usepackage{subcaption}
\usepackage{array}
\usepackage{physics}
\usepackage{enumitem}
\usepackage[svgnames]{xcolor}

\usepackage{newtxtext}
\usepackage{newtxmath}

\usepackage{amsthm}
\usepackage{thmtools}
\usepackage{thm-restate}

\theoremstyle{plain}
\newtheorem{theorem}{Theorem}[section]
\newtheorem{proposition}[theorem]{Proposition}
\newtheorem{lemma}[theorem]{Lemma}

\theoremstyle{definition}
\newtheorem{definition}[theorem]{Definition}
\newtheorem{assumption}[theorem]{Assumption}
\newtheorem{example}[theorem]{Example}
\theoremstyle{remark}
\newtheorem{remark}[theorem]{Remark}

\usepackage{tikz}
\usepackage{graphicx}
\usetikzlibrary{
    calc,
    positioning,
    arrows.meta
}

\usepackage{cleveref}
\crefname{assumption}{Assumption}{Assumptions}
\crefname{equation}{Equation}{Equations}
\crefname{theorem}{Theorem}{Theorems}
\crefname{lemma}{Lemma}{Lemmas}
\crefname{proposition}{Proposition}{Propositions}
\crefname{corollary}{Corollary}{Corollaries}
\crefname{definition}{Definition}{Definitions}
\crefname{note}{Note}{Notes}
\crefname{table}{Table}{Tables}
\crefname{figure}{Figure}{Figures}
\crefname{example}{Example}{Examples}
\crefname{section}{Section}{Sections}
\crefname{app}{Appendix}{Appendices}

\usepackage{pifont}
\newcommand{\cmark}{\ding{51}}
\newcommand{\xmark}{\ding{55}}

\usepackage{amsmath,amsfonts,bm}

\def\eqref#1{(\ref{#1})}

\def\1{\bm{1}}

\DeclareMathAlphabet{\mathsfit}{\encodingdefault}{\sfdefault}{m}{sl}
\SetMathAlphabet{\mathsfit}{bold}{\encodingdefault}{\sfdefault}{bx}{n}

\def\gV{{\mathcal{V}}}

\def\gX{{\mathcal{X}}}
\def\gY{{\mathcal{Y}}}

\newcommand{\R}{\mathbb{R}}

\newcommand{\KL}{D_{\mathrm{KL}}}

\usepackage{multirow}

\usepackage{etoolbox} 
\newbool{includeapp}
\setbool{includeapp}{true} 

\usepackage{siunitx}

\hypersetup{
  colorlinks = true,
  allcolors = DarkRed,
}

\definecolor{color1}{HTML}{0077BB}
\definecolor{color2}{HTML}{EE7733}
\definecolor{color3}{HTML}{33BBEE}
\definecolor{color4}{HTML}{EE3377}
\definecolor{color5}{HTML}{CC3311}
\definecolor{color6}{HTML}{009988}
\definecolor{color7}{HTML}{BBBBBB}

\usepackage[backend=biber,style=numeric,sorting=none,natbib=true]{biblatex}
\DeclareFieldFormat{urldate}{}

\DeclareSourcemap{
  \maps[datatype=bibtex]{
    \map{
      \step[fieldset=primaryclass, null]
    }
  }
}



\title{Equivariance and Augmentation for Bayesian Neural Networks}

\makeatletter
\def\@fnsymbol#1{\ensuremath{\ifcase#1\or a\or b\or *\or c\or d\or e\or f\or g \or h \else\@ctrerr\fi}}
\makeatother

\author{%
  Miaowen Dong\footnotemark[1] \and
  Axel Flinth\footnotemark[3]{\ \,}\footnotemark[2] \and
  Jan E.\ Gerken\footnotemark[3]{\ \,}\footnotemark[1]
}

\date{}

\begin{document}
\renewcommand{\thefootnote}{\fnsymbol{footnote}}
\maketitle
\footnotetext[3]{Equal Contribution}
\footnotetext[1]{Department of Mathematical Sciences, Chalmers University of Technology and the University of Gothenburg, SE-412 96 Gothenburg, Sweden.\\ Emails: miaowen@chalmers.se, gerken@chalmers.se}
\footnotetext[2]{Department of Mathematics and Mathematical Statistics, Umeå University, Linnéus väg 49, 901 87 Umeå, Sweden.\\ Email: axel.flinth@umu.se}
\renewcommand{\thefootnote}{\arabic{footnote}}
\setcounter{footnote}{0}

\begin{abstract}
  \vspace{-0.5ex}
Symmetries are important for many deep learning tasks, ranging from applications in the sciences to medical imaging. However, there is an ongoing debate about whether to impose symmetry constraints on the neural network architecture (yielding equivariant neural networks) or learn them from augmented training data. Although equivariant networks are well-studied theoretically, much less is known about data augmentation, since analyzing augmentation requires control over the training dynamics. Inspired by recent results that show that augmented infinite deep ensembles are exactly equivariant, we study data augmentation for Bayesian neural networks (BNNs) trained with variational inference. We focus on variational distributions in the exponential family and derive conditions under which exact equivariance is reached. We furthermore obtain bounds on the equivariance error and introduce three novel symmetrization techniques which boost the effect of data augmentation in this setting. We conduct extensive numerical experiments which show that one of our symmetrization methods (orbit expansion) outperforms the baseline in both equivariance and overall performance. Our code is available at \href{https://github.com/dmw1998/augment-BNNs}{\texttt{github.com/dmw1998/augment-BNNs}}.
\end{abstract}

\section{Introduction}
  \vspace{-0.5ex}
In recent years, symmetric learning tasks have become an important field of study. After an initial focus on specialized equivariant networks which impose the symmetry constraints layer-by-layer \citep{bronstein2017geometric}, attention has recently shifted towards learning symmetries from augmented training data~\citep{chen2020group}. This approach has the advantage of, given an efficient symmetry transformation mechanism,  being straightforward to implement
and can be used together with highly-optimized and well-performing architectures~\citep{wang2024}. However, as the symmetry is merely learned and not imposed, it is only achieved approximately.
Therefore, novel techniques are required which improve the symmetry gain derived from augmented training.

\begin{figure}[t]
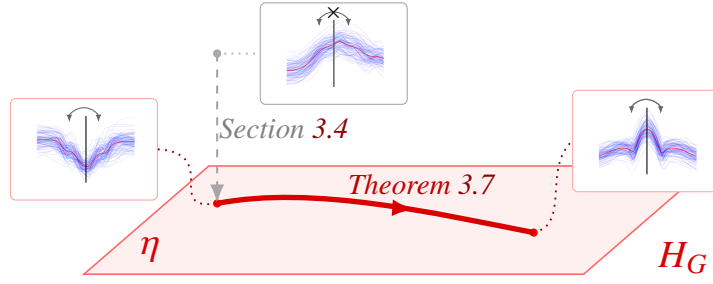

\centering
  \include{plots/tikz-picture}
  \vspace{-0.6cm}
  \caption{Natural parameters $\eta$ for a variational distribution in the exponential family that lie in $H_G$ correspond to symmetric BNNs, here exemplified with a reflection symmetry. Our main Theorem~\ref{thm:equivariance-of-vi} implies that $H_G$ is invariant for augmented training. Through the symmetrization strategies described in Section \ref{sec:symmetrization-of-the-variational-posterior}, we can increase the equivariance of the final model.  }
  \label{fig:figureone}
\end{figure}
The explicit layer-wise equivariant neural networks are readily accessible to theoretical analysis, while data augmentation is considerably harder to study since it involves the training dynamics (see the related works section).
It is however possible to show that in \emph{expectation over initializations}, data augmentation leads to exact equivariance \citep{gerken2024emergent,nordenfors2024}. A practical but costly method for approximating such expected values is training a deep ensemble.

The purpose of this work is to investigate a cheaper way to realize such ``equivariance in expectation'': Training Bayesian neural networks (BNNs) with variational inference on augmented data. In this setting, sampling from the posterior predictive distribution replaces the inference step on the ensemble and yields Bayesian uncertainty estimates. In this setting, only one training run is necessary to obtain the variational posterior, in contrast to the one training run per ensemble member
for deep ensembles. Furthermore, since BNNs show stable out of distribution behavior, they are particularly well suited for small datasets, the regime in which data augmentation can be expected to have the largest effect.

Our main contributions are the following:
\begin{itemize}
    \item We study BNNs trained on augmented data with a variational distribution from the exponential family. We show that, when the training starts from an invariant prior, the variational distribution stays invariant throughout training, under some mild assumptions. This generalizes a similar result from \citet{nordenfors2025optimization} for non-Bayesian network training.
    \item We derive bounds for the deviation of the variational distribution from equivariance if the prior is not equivariant. We furthermore prove bounds for the equivariance error in the predictions due to finite sampling. These theoretical results are validated numerically.
    \item We introduce three symmetrization methods (geometric averaging, projection and orbit expansion) that can be applied during training to improve the equivariance properties of the BNN. We test these techniques in extensive numerical experiments for image classification. \emph{Orbit expansion} outperforms the baseline in both model performance and equivariance.
    
\end{itemize}

\section{Related Work}
  \vspace{-0.5ex}
  
\paragraph{Equivariant neural networks.}
The issue of symmetry,
i.e. \emph{invariance} and \emph{equivariance}, of deep neural nets has grown into an entire  subfield called \emph{geometric deep learning} \citep{bronstein2017geometric}. The most prominent construction of
equivariant networks is the layerwise one. This strategy originates from GCNNs, Group Convolutional Neural Networks \citep{cohenGroupEquivariantConvolutional2016}, but has by now been generalized to virtually any symmetry induced by any group \citep{kondor2018generalization,keriven2019universal,gerken2023}. There are also other strategies, such as learning from invariants \citep{villar2021scalars}, frame averaging \citep{puny2022frame}, fundamental domain projection \citep{aslan2023group} and group averaging. There are also works in which symmetries are enforced approximately, for instance through so called weight annealing \citep{wang2022approximately}.

\paragraph{Augmentation and training dynamics.}
The question of the effect of data augmentation on the training dynamics of neural networks have been treated in several simplified contexts, such as feature averaged models \citep{lyle2020benefits,elesedy2021provably} as well as linear neural networks \citep{lawrence2022implicit,chen2024implicit,montufar2025equivariant}. In these cases, it is often possible to prove equivalence of augmentation and equivariance. Fully non-linear networks were treated in \citet{nordenfors2025optimization,nordenfors2025}, whose results we generalize to Bayesian networks here.

Empirical studies on the difference between augmentation and restrictions are plentiful. More systematic treatments are \citet{gandikotaTrainingArchitectureHow2021,gerken22a,brehmer2025}.

\paragraph{Bayesian neural networks.}
Bayesian approaches to deep learning have been studied for many years~\citep{mackay1992} since they offer uncertainty estimates for neural networks that are otherwise black-box models (see the PhD thesis by~\citet{Gal2016Uncertainty} for an overview).
However, making BNNs practically applicable requires integrating variational inference into the deep learning training methodology~\citep{gravesPracticalVariationalInference2011,blundell2015,kingma2022}.
For a review of BNNs with a focus on practical applications, see~\citet{jospin2022}.

Not many prior works consider symmetries in relation to BNNs.
\citet{pmlr-v180-ouderaa22a} propose a probabilistic group-averaging of BNNs in order to realize a soft symmetry constraint, optimized on the data. Closest to our work is~\citet{mourdoukoutas2021} which uses a specific prior that combines different weight-sharing schemes and therefore symmetry constraints. During training, the network learns which symmetry fits the data best. In contrast, we consider training on augmented data with a generic prior that does not impose weight-sharing.

\section{Theory}
  \vspace{-0.5ex}
\label{sec:theory}

We develop a theoretical framework for understanding how data augmentation induces equivariance in variational Bayesian inference. 
We proceed in three steps: first, we characterize when exponential families are closed under group actions (\cref{sec:closed-exponential-family-under-group-actions}); second, we show that data-augmented training
makes the ELBO invariant, and how that affects the training (\cref{sec:data-augmentation-induces-equivariance});
third, we propose
symmetrization mechanisms and analyze their properties (\cref{sec:symmetrization-of-the-variational-posterior}).

\subsection{Preliminaries}
\label{sec:preliminaries}

In this section, we introduce the mathematical tools used throughout this paper. We begin with exponential families, which provide the structural backbone for our theoretical analysis, before reviewing variational inference and the group theoretic concepts needed to formalize symmetry.

\paragraph{Exponential families.} An exponential family of probability distributions is defined through a base measure $h(\theta)$, a sufficient statistic $T(\theta)\in \R^k$ and a log-partition function $A(\eta)~:=~\log~\int~h(\theta) \exp(\eta^{\top} T(\theta)) d\theta$. A distribution belongs to the exponential family if its density has the form
\begin{equation}
    q_{\eta}(\theta) = h(\theta) \exp(\eta^{\top} T(\theta) - A(\eta))\,,
\end{equation}
where $\eta\in H \subseteq \R^k$ is the \emph{natural parameter}.
Examples of such families include normal, exponential and log-normal distributions.  An exponential family $\mathcal{Q} := \{q_{\eta}(\theta) \mid \eta \in H\}$,  is \emph{regular} if $H$ is open, and \emph{minimal} if the components of $T(\theta)$ are linearly independent.

\paragraph{Group actions and push-forward distributions.} We assume that a group $G$ acts on the input space $\mathcal{X}$, the output space $\mathcal{Y}$, and the parameter space $\Theta$ via representations $\rho_{\mathcal{X}}$, $\rho_{\mathcal{Y}}$, and $\rho_{\Theta}$ respectively.
Throughout, we assume that $\rho_{\Theta}$ is compatible with the data representations $\rho_{\mathcal{X}}$ and $\rho_{\mathcal{Y}}$ in the following sense:
\begin{equation} \label{eq:compatibility}
    f(\rho_\mathcal{X}(g)x;\theta) = \rho_\mathcal{Y}(g)f(x;\rho_\Theta(g)^{-1}\theta).
\end{equation}
Following the layered decomposition $\Theta = \bigoplus_{\ell} \Theta_{\ell}$ of a neural network's parameters, we assume that $\rho_{\Theta}$ acts per layer, i.e., $\rho_{\Theta} = \bigoplus_{\ell} \rho_{\Theta_{\ell}}$. There is a canonical construction of such a $\rho_\Theta$: One introduces representations $\rho_{\Theta_{\ell}}$ on the hidden layers and then imposes equivariance  on each linear layer with respect to the two representations acting on the input and output of that layer. As long as the non-linearities are equivariant with respect to the same representations (which is always the case for pointwise non-linearities and permutation representations) one obtains \eqref{eq:compatibility}. For more details, see \citet{nordenfors2025optimization}.
The choice of $\rho_{\Theta}$, or equivalently the choice of the hidden layer representations, is a modeling decision that we revisit in \cref{sec:symmetrization-of-the-variational-posterior}. When there is no risk of confusion, we write $gx$ and $gy$ as shorthands for $\rho_{\mathcal{X}}(g)x$ and $\rho_{\mathcal{Y}}(g)y$, but retain the explicit notation $\rho(g)\theta$ for the parameter space, where the representation structure is central to our analysis. The push-forward of a distribution $q$ under $g \in G$ is defined as
\begin{equation}
    \mathcal{T}_{g} \# q(B) := q(\rho(g)^{-1} B)
\end{equation}
for any measurable set $B$, with density
   $ \mathcal{T}_{g} \# q (\theta) = q(\rho(g)^{-1} \theta) \left|\det D(\rho(g)^{-1}) (\theta)\right|$.
A distribution $q$ is called $G$-invariant if $\mathcal{T}_{g} \# q = q$ for all $g \in G$.

\paragraph{Bayesian neural networks (BNN) and variational inference (VI).} A neural network is a parametrized function $f(\cdot; \theta): \mathcal{X} \rightarrow \mathcal{Y}$. They define conditional label distributions $p(y\, \vert \, x, \theta) = p(y\, \vert \, f(x;\theta))$. In the Bayesian treatment, we put a prior $p_{0}(\theta)$ on the weights and calculate the posterior $p(\theta \mid \mathcal{D})$ given a dataset $\mathcal{D} = \{(x_{i}, y_{i})\}_{i=1}^{N_{0}}$ via Bayes' rule. However, since the posterior neural network likelihood typically is intractable, one makes a variational inference approximation $q_\eta \in \mathcal{Q}$, and then choose the parameters $\eta$  to maximize the \emph{evidence lower bound} \citep{bishopPatternRecognitionMachine2006}
\begin{equation}
    \mathrm{ELBO}(\eta) := 
    \mathbb{E}_{q_{\eta}}
    [\log p(\mathcal{D} \mid \theta)] 
    - \KL(q_{\eta} \Vert p_{0})\,.
    \label{eq:elbo}
\end{equation}
The posterior predictive is approximated by Monte Carlo sampling from $q_\eta$:
\begin{equation}
    F_\eta(x) := \mathbb{E}_{\theta \sim q_{\eta}} [f(x;\theta)] \approx \frac{1} {T}\sum_{t=1}^{T} f(x;\theta^{(t)})\,, \quad \theta^{(t)} \sim q_\eta\,.
    \label{eq:post-predictive}
\end{equation}

\paragraph{Data augmentation.} Given a finite group $G$ and a dataset $\mathcal{D} = \{(x_{i}, y_{i})\}_{i=1}^{N_{0}}$ drawn i.i.d. from a data distribution $P_{\mathcal{X}}$, data augmentation constructs an augmented dataset $\mathcal{D}_{\mathrm{aug}} = \{(gx_{i}, gy_{i}) \mid g \in G, (x_{i},y_{i}) \in \mathcal{D}\}$, expanding the dataset by a factor of $|G|$.
One could also consider a  continuous compact group
by sampling finitely many Haar-distributed group elements.
This comes with some subtleties that we discuss in \cref{app:continuous-compact-groups}. In the following, we instead restrict to the case of finite $G$. Following the mini-batch scaling convention of \citet{gravesPracticalVariationalInference2011} and \citet{blundell2015}, instead of using the ELBO~\eqref{eq:elbo} directly, we work with a $\beta$-weighted objective:
\begin{equation}
    \mathcal{L}_{\beta}(\eta) := \underbrace{\tfrac{1}{N}\,\mathbb{E}_{q_{\eta}}[-\log p(\mathcal{D}_{\mathrm{aug}} \mid \theta)]}_{=:\,R_{\mathrm{aug}}(\eta)} + \beta \KL(q_{\eta} \Vert p_{0}),
    \label{equ:L-beta}
\end{equation}
where $N = |G| \cdot N_{0}$ is the augmented dataset size. $\mathcal{L}_{\beta}$ with $\beta = 1/N$ is the rescaled negative ELBO.

If we assume the following invariance of the forward model: 
    \begin{equation} \label{eq:forward_invariance}
        p(gy \, \vert \, gf(x;\theta)) = p(y\, \vert \, f(x;\theta)),
    \end{equation}
 the compatibility of $\rho_{\Theta}$ with the data representation implies that the augmented likelihood inherits the symmetry of the data. Note that condition \eqref{eq:forward_invariance} is weak -- it is e.g.  satisfied for $y\sim f(x;\theta) +e$ for a $G$-invariant error $e$ -- and furthermore independent of the architecture $f$.

\begin{restatable}[Invariant augmented likelihood]{proposition}{restateinvariantlikelihood}
\label{prop:invariant-augmented-likelihood}
    Let $L_{\mathrm{aug}}(\theta) := p(\mathcal{D}_{\mathrm{aug}}\mid\theta)$ be the likelihood of the augmented dataset. Under the compatibility assumption \eqref{eq:compatibility}
    as well as the  invariance assumption~\eqref{eq:forward_invariance},
    the likelihood is $G$-invariant:
    \begin{equation}
        L_{\mathrm{aug}}(\rho(g)^{-1}\theta)=L_{\mathrm{aug}}(\theta),\, \forall g \in G,\, \theta \in \Theta\,.
    \end{equation}
\end{restatable}

The proof is given in \cref{app:proof-of-prop-invariant-augmented-likelihood}.

\paragraph{Equivariance.}  A single  network $f(\cdot\,;\theta)$ is $G$-equivariant if $f(gx;\theta) = g f(x;\theta)$ for all $x \in \mathcal{X}$, $g \in G$. More generally, the predictive map $F_\eta$ defined in~\eqref{eq:post-predictive} is \emph{$G_\varepsilon$-equivariant} if the equivariance defect \mbox{$\Delta^{\mathrm{eq}}_{F}(\eta)$} is bounded by $\varepsilon$,
\begin{equation}
    \Delta^{\mathrm{eq}}_{F}(\eta) := \mathbb{E}_{x \sim P_{\mathcal{X}},\, g \sim \nu_{G}}\!\left[\left\|F_{\eta}(gx) - g F_{\eta}(x)\right\|^{2}\right] \leq \varepsilon.
\end{equation}
Here, $\nu_G$ denotes the normalized Haar measure on $G$. When $\varepsilon = 0$, this reduces to exact $G$-equivariance of $F_\eta$ $P_\mathcal{X}$-a.s. One central question is whether minimizing $\mathcal{L}_\beta$ on $\mathcal{D}_\mathrm{aug}$ drives $\Delta^\mathrm{eq}_F(\eta)$ to zero, and at what rate.

\subsection{Closed exponential family under group actions}
\label{sec:closed-exponential-family-under-group-actions}

Neural networks can learn equivariance efficiently from augmented data only when the parameter space is closed under group transformations \citep{nordenfors2025optimization}. Similarly, in the Bayesian setting, we restrict the variational family to those closed under group transformations. As the following theorem shows, this imposes constraints on the exponential family.
We provide the proof in \cref{app:proof-of-thm-closure-under-push-forward}.

\begin{restatable}[Closure under push-forward]{theorem}{restateclosure}
\label{thm:closure-under-push-forward}
A regular minimal exponential family $\mathcal{Q}$ is closed under the push-forward operator $\mathcal{T}_{g}$ for all $g \in G$ if and only if for each $g$ there exist $M_{g} \in \mathbb{R}^{k \times k}$, $d_{g} \in \mathbb{R}^{k}$, $b_{g} \in \mathbb{R}^{k}$ and $c_{g} \in \mathbb{R}$ such that for all $\theta \in \Theta$,
\begin{align}
    T(\rho(g)^{-1} \theta) &= M_{g} T(\theta) + d_{g}, \\
    h(\rho(g)^{-1} \theta) \left|\det D(\rho(g)^{-1})(\theta)\right| &= h(\theta) \exp(b_{g}^\top T(\theta) + c_{g}). \label{eq:base_measure}
\end{align}
When these conditions hold, $\mathcal{T}_{g} \# q_{\eta} = q_{\eta_{g}}$ where $\eta_{g} = M_{g}^\top\eta + b_{g}$, and the transformation  $\phi_{g}: \eta\mapsto M_{g}^\top\eta + b_{g}$ forms an affine action of $G$ on the space of natural parameters $H$.
\end{restatable}
A family $\mathcal{Q}$ is hence closed if and only if there exists an affine representation of the group acting on the ambient space $\R^k$ of the space of natural parameters, that makes $T: \Theta \to \R^k$ invariant, and with respect to which the base measure transforms according to \eqref{eq:base_measure}.

\begin{example}[Mean-field Gaussian under permutation actions]
    Consider the mean-field Gaussian family $\mathcal{Q} = \{\mathcal{N}(\mu, \mathrm{diag}(\sigma^{2})): \mu \in \mathbb{R}^{d}, \sigma^{2} \in \mathbb{R}^{d}_{>0}\}$ on $\Theta = \mathbb{R}^{d}$. This family has sufficient statistic $T(\theta) = (\theta,\,\theta \odot \theta) \in \mathbb{R}^{d}\oplus \mathbb{R}^d$, constant base measure $h(\theta) = (2 \pi)^{-d/2}$, and natural parameter $\eta = (\eta_{1}, \eta_{2})\in\mathbb{R}^{d}\oplus \mathbb{R}^d$ given by $(\eta_{1})_{i} = \mu_{i} / \sigma_{i}^{2}$ and $(\eta_{2})_{i} = -1/(2 \sigma_{i}^{2})$. 
    
    Let $G$ act on $\Theta$ by a permutation representation $\rho(g) = R_{g}$. Then, both conditions of \cref{thm:closure-under-push-forward} are met: Permutation matrices commute with the Hadamard product, so $T(R_{g}^{-1} \theta) = (R_{g}^{-1} \oplus R_{g}^{-1})T(\theta)$
    giving $M_{g} = R_{g}^{-1} \oplus R_{g}^{-1}$ and $d_{g} = 0$. Since $|\det R_{g}^{-1}| = 1$ and $h$ is constant, \eqref{eq:base_measure} is fulfilled by $b_{g} = c_{g} = 0$. Using $(R_{g}^{-1})^{\top} = R_{g}$, the induced natural-parameter transformation is
          \vspace{-0.5ex}
    \begin{equation}
        \phi_{g}(\eta) = M_{g}^{\top} \eta = (R_{g} \oplus R_{g}) \eta,
    \end{equation}
    which simply amounts to permuting the entries of $\eta_{1}$ and $\eta_{2}$. Equivalently,
          \vspace{-0.5ex}
    \begin{equation}
        \mathcal{T}_{g} \# \mathcal{N}(\mu, \mathrm{diag}(\sigma^{2})) = \mathcal{N}(R_{g} \mu, \mathrm{diag}(R_{g} \sigma^{2})),
    \end{equation}
    so the family is indeed closed.
\end{example}

\subsection{Data augmentation induces equivariance}
\label{sec:data-augmentation-induces-equivariance}

The representation $\phi_g$ fixes the set of  symmetric parameters
$H_G : = \{\eta \in H \, \vert \, \phi_g(\eta)=\eta, \, \forall \, g\in G\}$.
Just as in the single network case, parameters in $H_G$ correspond to equivariant predictive maps.

\begin{proposition}
    For $\eta\in H_G$, $F_\eta$ is $G$-equivariant.
\end{proposition}
\begin{proof}
    If $\eta\in H_G$, $\mathcal{T}_{g} \# q_{\eta} =q_{\phi_{g}(\eta)}=q_{\eta}$. Hence, if $\theta \sim q_\eta$, $\rho(g)\theta\sim q_{\eta}$ for every $g\in G$. This together with the compatibility assumption implies that $f(gx;\theta) =gf(x;\rho(g^{-1})\theta)\stackrel{\mathrm{d}}{=} gf(x;\theta)$, where the last equality is in distribution. Taking the expectation over $\theta$ yields $F_\eta(gx)=gF_\eta(x)$, which is the claim. 
\end{proof}

Our main result states that, under the conditions presented below, $H_{G}$ is invariant under gradient descent on augmented data. Initializing $\eta$ there guarantees equivariant $F_\eta$ throughout training.

\begin{assumption}[Invariant prior]
\label{ass:invariant-prior}
    The prior $p_{0} = q_{\eta_{0}}(\theta)$ for some $\eta_{0}$ is $G$-invariant, i.e., $\mathcal{T}_{g} \# p_{0} = p_{0}$ for all $g \in G$, or equivalently $\phi_{g}(\eta_{0}) = \eta_{0}$.
\end{assumption}

\begin{assumption}[Volume-preserving action]
\label{ass:volume-preserving}
    $|\det D(\rho(g)^{-1})(\theta)| = 1\,$ for all $g \in G$ and $\theta \in \Theta$.
\end{assumption}

 Both these assumptions are natural: an invariant prior
 matches the hypothesis of a $G$-symmetric task,
  and volume-preservation is automatic for the, in practice very common, orthogonal actions.

\begin{restatable}[Equivariance of VI]{theorem}{restateequivariance}
\label{thm:equivariance-of-vi}
    \cref{ass:invariant-prior,ass:volume-preserving} imply the following for variational inference using  invariant likelihoods and closed exponential families:
    \begin{enumerate}[label=(\roman*)]
        \item \textbf{Loss invariance.} The loss $\mathcal{L}_{\beta}$ is invariant under the natural parameter transformation:
              \vspace{-0.5ex}
        \begin{equation}
            \mathcal{L}_{\beta}(\eta) = \mathcal{L}_{\beta}(\phi_{g}(\eta)), \quad \forall g \in G,\, \eta \in H.
        \end{equation}
        \item \textbf{Gradient equivariance and update commutativity.} The gradient of the loss $\mathcal{L}_{\beta}$ satisfies
              \vspace{-0.5ex}
        \begin{equation}
            \nabla_{\eta} \mathcal{L}_{\beta}(\phi_{g}(\eta)) = M_{g}^{-1} \nabla_{\eta} \mathcal{L}_{\beta}(\eta), \quad \forall g \in G,\, \eta \in H.
        \end{equation}
        Consequently, for $M_{g}$ orthogonal, gradient descent $\mathcal{U}(\eta) = \eta - \alpha \nabla_{\eta} \mathcal{L}_{\beta}(\eta)$ commutes with $\phi_{g}$:
              \vspace{-0.5ex}
        \begin{equation} \label{eq:update_commute}
            \mathcal{U}(\phi_{g}(\eta)) = \phi_{g}(\mathcal{U}(\eta)).
        \end{equation}
        \item \textbf{Invariant subspace is preserved.} For $M_{g}$ orthogonal, the equivariant subspace $H_{G}$
        is preserved by gradient updates: if $\eta^{(0)} \in H_{G}$, then $\eta^{(t)} \in H_{G}$ for all $t \geq 0$.
        \item \textbf{Optimal solutions inherit symmetry.} The set of optimal natural parameters $H^{\ast} = \arg\min_{\eta} \mathcal{L}_{\beta}(\eta)$ is closed under $\phi_{g}$ for all $g \in G$. If the optimum is unique, then $\phi_{g}(\eta^{\ast}) = \eta^{\ast}$.
    \end{enumerate}
\end{restatable}

The proof is given in \cref{app:proof-of-thm-equivariance-of-vi}. This formalizes the intuition behind data augmentation: by making the training objective symmetric, augmentation ensures that $H_G$ is invariant -- and in the case of a unique minimum, even that the minimum is in $H_G$.

Note that Theorem~\ref{thm:equivariance-of-vi} does not rely on augmented data directly, only on an invariant likelihood (which is implied by data augmentation according to Proposition~\ref{prop:invariant-augmented-likelihood}). Therefore, the theorem is still applicable when the symmetry is present in the data but unknown.
However in such a case, it is hard to guarantee that the prior is invariant. The next theorem, which we prove in \cref{app:proof-of-thm-prior-independent-convergence}, shows that as the dataset size increases, the effect of a non-invariant prior vanishes.

\begin{restatable}[Prior-independent convergence]{theorem}{restateconvergence}
\label{thm:prior-independent-convergence}
Assume \cref{ass:volume-preserving}, and that the likelihood is invariant. Let $R^{\ast}_{\mathrm{aug}}:= \inf_{\eta}R_{\mathrm{aug}}(\eta)$ and $\eta^{\ast}_{\beta}(p_{0})\in\arg\min_{\eta}\mathcal{L}_{\beta}(\eta; p_{0})$.
We then have 
\begin{equation}
    R_{\mathrm{aug}}(\eta^{\ast}_{\beta}(p_{0})) 
    \leq R^{\ast}_{\mathrm{aug}} 
    + \beta \cdot C(p_{0}),
\end{equation}
for any prior $p_{0}$, where $C(p_{0}) := \inf_{\eta^{\ast} \in H^{\ast}_{R}} \KL(q_{\eta^{\ast}} \Vert p_{0})$ and $H^{\ast}_{R} := \arg\min_{\eta} R_{\mathrm{aug}}(\eta)$ is the set of risk-optimal natural parameters. As $\beta \to 0$, all minimizers achieve $R^{\ast}_{\mathrm{aug}}$ regardless of the prior.
\end{restatable}

Note that the quantity $C(p_{0})$ measures how well the prior aligns with the risk-optimal orbit. Since $\beta = 1/N$, with sufficient data, it becomes irrelevant. In contrast, $R_{\mathrm{aug}}$ is asymptotically constant; It will almost surely converge to the true likelihood $\mathbb{E}_{\substack{(x,y)\sim P_{\mathcal{X}}}}\mathbb{E}_{g\sim \nu_G}\mathbb{E}_{q_{\eta}}[-\log p((gx,gy) \mid \theta)]$.

Next, we address the fact that the true predictive maps $F_\eta$ in practice must be approximated with finite sample means. This will indeed cause a Monte Carlo equivariance defect $\widehat{\Delta}^{\mathrm{eq}}_{F}(\eta)$ (defined formally in~\eqref{equ:def-delta-hat}), but that defect vanishes with a number of samples grows, at the expected rate, as the following theorem shows (proof in \cref{app:proof-of-thm-complexity-for-G-e-equivariance}). We verify it and \cref{thm:prior-independent-convergence} in \cref{sec:validation-of-theorems}.

\begin{restatable}[Complexity for $G_{\varepsilon}$-equivariance]{theorem}{restatecomplexity}
\label{thm:complexity-for-G-e-equivariance}
 Suppose $f(x;\theta)$ is uniformly bounded in $x \in \mathcal{X}$ and 
    $q_{\eta}$-a.s.\ in $\theta$, and that the $G$-action on $\mathcal{Y}$ 
    is orthogonal. Then there exist constants $C_{\mathrm{data}}$ and $C_{\mathrm{weight}}$ such that, for any fixed $\eta$, with probability at least $1-\delta$ over the $N_0$ i.i.d.\ inputs $x_i \sim P_\mathcal{X}$ and the $T$ i.i.d.\ posterior samples $\theta^{(t)} \sim q_\eta$, the Monte Carlo equivariance defect satisfies
    \begin{equation}
        \widehat{\Delta}^{\mathrm{eq}}_{F}(\eta) \leq  \Delta^{\mathrm{eq}}_{F}(\eta)+ C_{\mathrm{data}}\sqrt{\frac{\log(2/\delta)}{N_{0}}} + C_{\mathrm{weight}}\sqrt{\frac{\log(2/\delta)}{T}}\,.
    \end{equation}
   In other words, for any $\varepsilon_{0},\varepsilon_{\mathrm{data}}, \varepsilon_{\mathrm{weight}} >0$, if $\Delta_F^{\mathrm{eq}}(\eta)<\varepsilon_0$, it suffices to choose
   \begin{equation}
        N_{0} \ge \frac{C_{\mathrm{data}}^{2}}{\varepsilon_{\mathrm{data}}^{2}}\log\left(\frac{2}{\delta}\right), \quad T \ge \frac{C_{\mathrm{weight}}^{2}}{\varepsilon_{\mathrm{weight}}^{2}}\log\left(\frac{2}{\delta}\right)
    \end{equation}
    in order for
    $\widehat{\Delta}_F^{\mathrm{eq}}(\eta) < \varepsilon_0 + \varepsilon_{\mathrm{data}}+\varepsilon_{\mathrm{weight}}$ with probability at least $1-\delta$.
\end{restatable}

\subsection{Symmetrization of the variational posterior}
\label{sec:symmetrization-of-the-variational-posterior}

According to \cref{thm:equivariance-of-vi}(iii), the space of parameters corresponding to equivariant variational distributions $H_{G}$ is preserved during training. Hence, if we ensure that the parameters lie in $H_{G}$ at some point during training, we will arrive at an equivariant variational posterior at the end of it. Therefore, we propose two complementary mechanisms that place $\eta$ into $H_{G}$, one by projecting an existing posterior, one by constructing the posterior from a smaller base, see Figure~\ref{fig:figureone}.

\paragraph{Orbit averaging.} Given a posterior distribution $q_{\eta}$ at any point during training, we construct a $G$-invariant posterior by averaging its natural parameter over the orbit:
\begin{equation}
    \tilde{\eta} = \frac{1}{|G|} \sum_{g \in G} \phi_g(\eta).
    \label{equ:orbit-average}
\end{equation}
We can interpret this operation in two ways:
First, $\tilde{\eta}$ is the natural parameter of the \emph{geometric mean distribution} $\tilde{q}(\eta) \propto \left[\prod_{g \in G} \Big(\mathcal{T}_{g} \# q_{\eta}\Big)(\theta)\right]^{\frac{1}{|G|}}$ (which lies in $\mathcal{Q}$ by \cref{lmm:geometric-mean-invariance-and-membership} in \cref{app:geometric-averaging-stays-in-Q}). Secondly, for orthogonal $\phi_g$, $\tilde{\eta}$ is the orthogonal projection of $\eta$ onto the invariant subspace $H_{G}$ under the Euclidean inner product -- this in particular shows that $\tilde{\eta}\in H_G$.

The mechanism specializes through the choice of $\rho_{\Theta}$. Let us exemplify this in the case of $C_4$ acting on
convolutional layers: First, we can define a representation  acting by rotating each individual convolutional channel  and secondly a representation that additionally permutes channel blocks, inspired by group convolutional neural  (GCNNs) from \citet{cohenGroupEquivariantConvolutional2016}. The corresponding average operations on the filters are depicted in
Figure \ref{fig:symmetrization-panels-avg}. We will refer to the former as \textbf{geometric averaging}, and the latter as \textbf{projection}. We derive both in detail in \cref{app:symmetrization-strats}.

\begin{figure}
    \centering
    \includegraphics[width=1\linewidth, trim=0.4cm 1.7cm 0.6cm 1.49cm, clip]{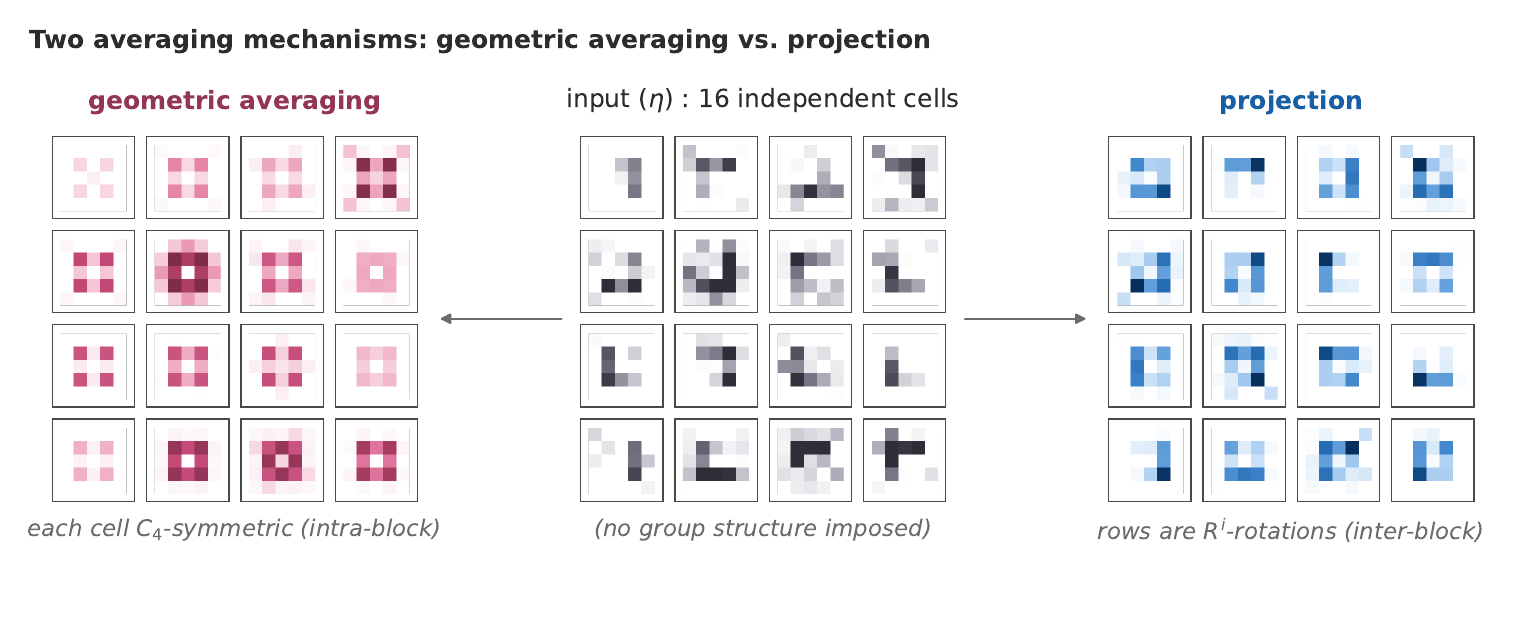}
    \caption{Two specializations of the orbit averaging \eqref{equ:orbit-average} on the second convolutional layer of a $C_{4}$-equivariant network. Both start from the same $4 \times 4$ block of variational parameters $\eta$.
    }
    \label{fig:symmetrization-panels-avg}
\end{figure}

\paragraph{Orbit expansion.} Instead of symmetrizing a general parameter of the intended size, we can also first train a small parameter and then, in a manner inspired by the weight-sharing structure of \\ GCNNs, extend it to an equivariant one of bigger size. Concretely, a base network with $1/|G|$ of the target channel width is trained in Stage 1, yielding a posterior $q_{\eta_{\mathrm{small}}}$. At the start of Stage 2, the network is expanded to full width by inserting rotated versions of each filter in $|G|\times |G|$-blocks (see \cref{fig:symmetrization-panels-gcnn}). We refer to this strategy as \textbf{orbit expansion}.
The resulting posterior $q_{\eta_{\mathrm{full}}}$ has a $G$-equivariant block structure at the parameter level. In \cref{app:expansion-along-orbit-filter-arrangement-ablation}, we compare three different expansion strategies.

\begin{figure}
    \centering
   \includegraphics[width=1\linewidth, trim=0.2cm 0.9cm 0.8cm 1.0cm, clip]{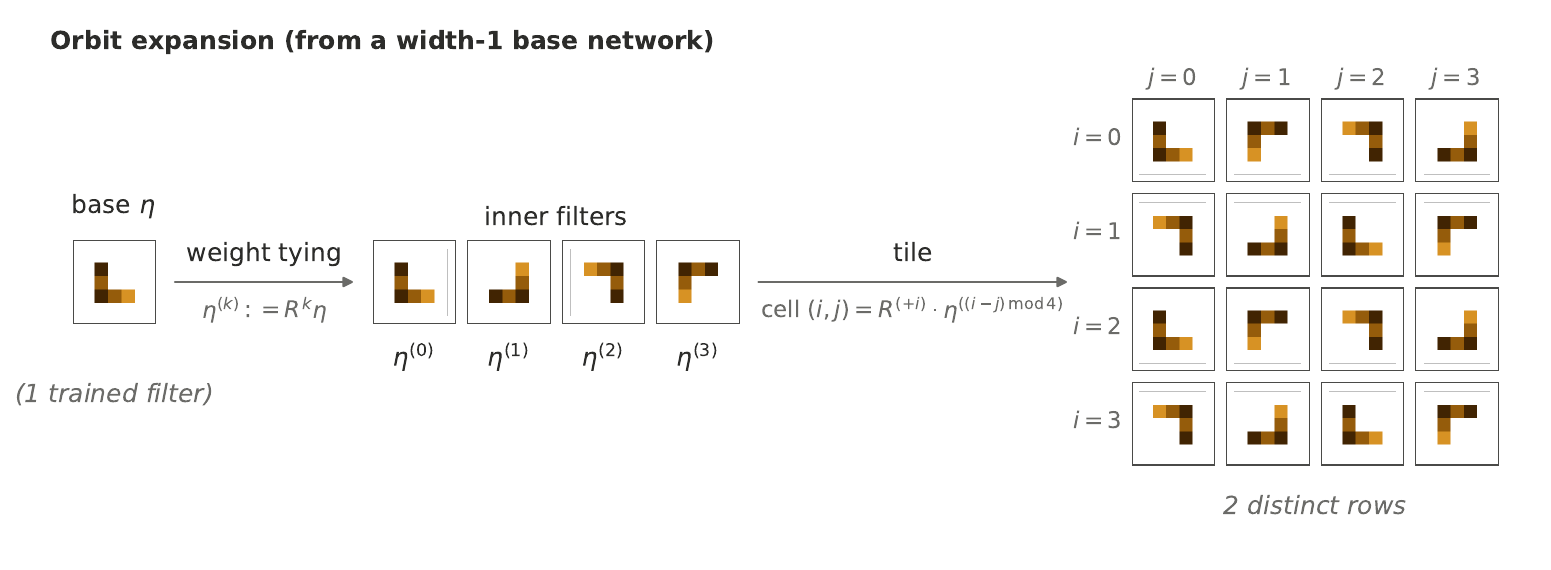}
    \caption{Filter arrangement for orbit expansion under $C_{4}$, starting from a single base filter $\eta$.}
    \label{fig:symmetrization-panels-gcnn}
\end{figure}

\section{Experiments}
In our experiments, we consider the cyclic group $C_4$ of rotations of multiples of $90^{\circ}$
acting on FashionMNIST \citep{xiao2017fashionmnistnovelimagedataset}. This setup allows us to augment exactly and 
enables extensive Monte Carlo sampling. Throughout our experiments, we use Gaussian variational distributions which satisfy the constraints of Theorem~\ref{thm:closure-under-push-forward}. In Appendix~\ref{app:variational-family-closure-ablation}, we provide additional ablations for other variational distributions. Code is provided at \href{https://github.com/dmw1998/augment-BNNs}{\texttt{github.com/dmw1998/augment-BNNs}}.
\subsection{Validation of Theorems}
\label{sec:validation-of-theorems}

We start by verifing the predictions of \cref{thm:prior-independent-convergence,thm:complexity-for-G-e-equivariance}. We use a Bayesian neural network with two convolutional layers and one linear layer and
vary the dataset size $N_{0} \in \{500,\,5000,\,20000,\,50000\}$. For \cref{thm:prior-independent-convergence}, we compare an invariant isotropic Gaussian prior $p_{0} = \mathcal{N}(0,I)$ against a non-invariant prior with means drawn from a normal distribution. Posterior predictives are estimated with $T$ Monte Carlo samples. The results are presented in \cref{fig:theorem-validation}.
\begin{figure}
    \centering
    \includegraphics[width=1\linewidth]{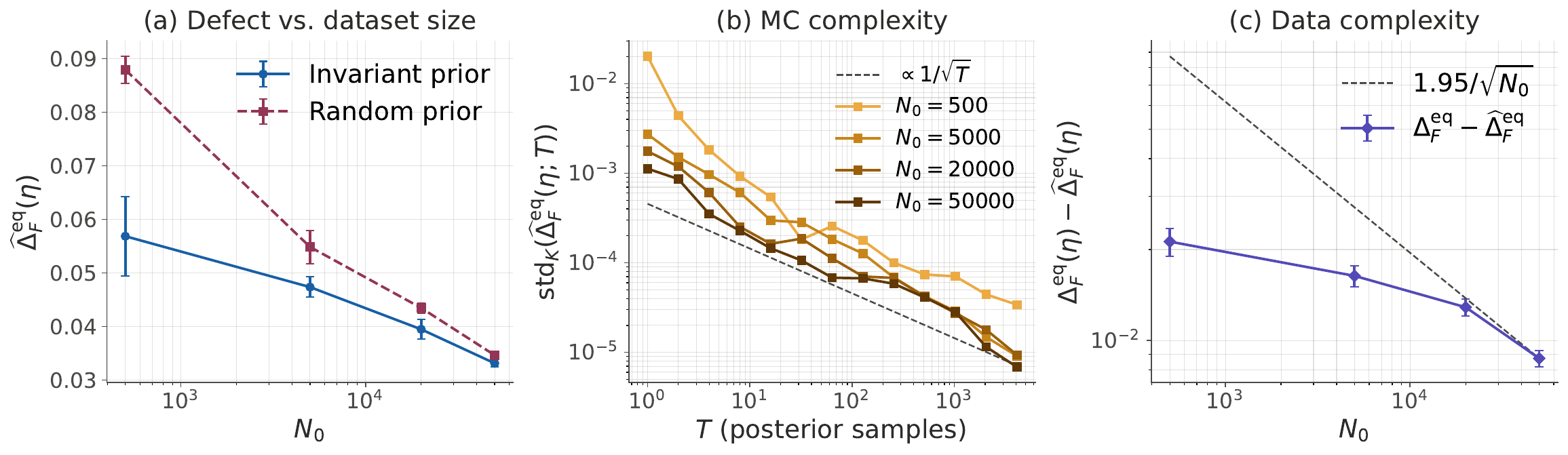}
    \caption{Empirical validation of \cref{thm:prior-independent-convergence,thm:complexity-for-G-e-equivariance}.
    \textbf{(a)} The empirical equivariance defect decreases with $N_{0}$ under both invariant and random Gaussian priors, and the two curves converge.
    \textbf{(b)} K-fold standard deviation of the Monte Carlo estimate $\widehat{\Delta}^{\mathrm{eq}}_F(\eta; T)$ across $K=10$ independent runs at each $T$. The decay matches the $\mathcal O(1/\sqrt T)$ rate predicted by \cref{thm:complexity-for-G-e-equivariance} (dotted reference line, slope $-0.5$).
    \textbf{(c)} The training gap decreases with $N_{0}$ in close agreement with the $1/\sqrt{N_{0}}$ rate predicted by \cref{thm:complexity-for-G-e-equivariance}. The dashed line is a fit with empirical constant $C_{\mathrm{data}} \approx 1.95$. Error bars are $\pm 1$ s.d. across 5 seeds.}
    \label{fig:theorem-validation}
\end{figure}

\paragraph{Priors.} The empirical defect $\widehat{\Delta}^{\mathrm{eq}}_{F}$ decreases monotonically with $N_{0}$ under both priors, from $0.057 \pm 0.007$ to $0.033 \pm 0.001$ (invariant) and $0.088 \pm 0.003$ to $0.035 \pm 0.001$ (random); the two curves converge as $N_{0}$ grows,
consistent with \cref{thm:prior-independent-convergence}.
\paragraph{Monte Carlo complexity.} \cref{thm:complexity-for-G-e-equivariance} predicts that the gap between the $T$-sample Monte Carlo estimate $\widehat{\Delta}^{\mathrm{eq}}_{F}(\eta;T)$ and the true predictive map decays as $\mathcal{O}(1/\sqrt{T})$. We assess this directly by drawing $K=10$ independent $T$-sample estimates at each $T$ and tracking their standard deviation across the $K$ runs. The results are shown in \cref{fig:theorem-validation}(b). The theorem is indeed verified: a $\log$-$\log$ fit of the $K$-fold std  against $T$
yields slopes of $-0.45$, $-0.66$, $-0.52$, $-0.51$ for $N_{0} \in \{500,\,5000,\,20000,\,50000\}$ respectively, all close to the predicted $-0.5$.

\paragraph{Data complexity.} In Figure \ref{fig:theorem-validation}(c), we use $T_{\max} = 1024$ as a proxy for $T \to \infty$ and plot the training gap $\Delta^{\mathrm{eq}}_{F}(\eta) - \widehat{\Delta}^{\mathrm{eq}}_{F}(\eta)$ together with $\mathcal{O}(1 / \sqrt{N_{0}})$ envelope from \cref{thm:complexity-for-G-e-equivariance}, with $C_{\mathrm{data}} \approx 1.95$. The measured gap decreases with $N_0$ in close agreement with the predicted $1/\sqrt{N_0}$ behavior.

\subsection{Symmetrization on image classification}
\label{sec:symmetrization-on-image-classification}

Next, we empirically compare the symmetrization mechanisms from \cref{sec:symmetrization-of-the-variational-posterior}.

\paragraph{Setup.} We train a convolutional Bayesian neural network with two convolutional layers with 32 and 64 channels, followed by a classification layer using a diagonal Gaussian variational family and a standard isotropic Gaussian prior. The training set consists of 5\,000 randomly chosen FashionMNIST images, each augmented with the full $C_{4}$ orbit (producing 20\,000 training samples). 

The parameter space of each convolutional layer decomposes along the channel dimension. We implement the orbit-averaging mechanism of \cref{sec:symmetrization-of-the-variational-posterior} under two natural choices of $\rho_{\Theta}$ introduced there: \textbf{geometric averaging} (per-filter $C_{4}$-symmetrization, channel arrangement unconstrained) and \textbf{projection} (averaging across filters and channel-block permutation). Both satisfy the closure conditions of \cref{thm:closure-under-push-forward}, so the results of \cref{sec:theory} apply uniformly. \textbf{Orbit expansion} is implemented with the GCNN-expansion from Figure \ref{fig:symmetrization-panels-gcnn}.
All methods share the same architecture, prior, optimizer, learning rate, epoch count, and number of MC samples; they differ only in when and how the posterior is symmetrized. We use SGD as the optimizer throughout the main text: its update rule is linear in the gradient and therefore commutes exactly with $\phi_{g}$ (\cref{thm:equivariance-of-vi}(ii)), so the equivariant subspace $H_{G}$ is preserved under training.

Each symmetrization mechanism is applied at epoch $n \in \{0, 20, 100\}$ to study the effect of trigger timing (\cref{sec:symmetrization-of-the-variational-posterior}). After the intervention, all methods continue training until epoch 500. As a baseline, we train a randomly initalized network on augmented data.
\paragraph{Metrics.} We report: (1) classification accuracy (\%), (2) orbit same prediction (OSP): the fraction of predictions that agree with the mode on the orbit, averaged over the test set, and (3) symmetric KL divergence between predictive distributions on original and rotated inputs.

\begin{table}[t]
    \centering
    \small
    \setlength{\tabcolsep}{12pt}
    \begin{tabular}{ll ccc}
    \toprule
    \textbf{Method} & \textbf{Epoch}
      & \textbf{Acc.\,(\%)} $\uparrow$
      & \textbf{OSP} $\uparrow$
      & \textbf{Sym.\,KL} $\downarrow$ \\
    \midrule
    Baseline & ---
      & $79.0{\scriptstyle\pm0.6}$
      & $0.927{\scriptstyle\pm.003}$
      & $0.0217{\scriptstyle\pm.0028}$ \\
    \midrule
    \multirow{3}{*}{Geo. avg.}
      & 0   & $78.6{\scriptstyle\pm0.4}$ & $0.950{\scriptstyle\pm.003}$ & $0.0094{\scriptstyle\pm.0009}$ \\
      & 20  & $78.6{\scriptstyle\pm0.8}$ & $0.952{\scriptstyle\pm.005}$ & $0.0094{\scriptstyle\pm.0020}$ \\
      & 100 & $78.0{\scriptstyle\pm1.1}$ & $0.946{\scriptstyle\pm.005}$ & $0.0112{\scriptstyle\pm.0020}$ \\
    \midrule
    \multirow{3}{*}{Proj.}
      & 0   & $78.4{\scriptstyle\pm0.8}$ & $0.931{\scriptstyle\pm.004}$ & $0.0174{\scriptstyle\pm.0026}$ \\
      & 20  & $78.5{\scriptstyle\pm0.2}$ & $0.931{\scriptstyle\pm.004}$ & $0.0180{\scriptstyle\pm.0026}$ \\
      & 100 & $77.2{\scriptstyle\pm0.9}$ & $0.929{\scriptstyle\pm.012}$ & $0.0182{\scriptstyle\pm.0046}$ \\
    \midrule
    \multirow{3}{*}{Orbit exp.}
      & 0   & $80.4{\scriptstyle\pm0.4}$ & $0.963{\scriptstyle\pm.003}$ & $0.0057{\scriptstyle\pm.0004}$ \\
      & 20 & $\mathbf{80.4}{\scriptstyle\pm0.6}$ & $0.966{\scriptstyle\pm.004}$ & $0.0048{\scriptstyle\pm.0009}$ \\
      & \cellcolor{gray!15}100 & \cellcolor{gray!15}$80.0{\scriptstyle\pm0.5}$ & \cellcolor{gray!15}$\mathbf{0.968}{\scriptstyle\pm.003}$ & \cellcolor{gray!15}$\mathbf{0.0042}{\scriptstyle\pm.0007}$ \\
    \bottomrule
    \end{tabular}
    \caption{Effect of symmetrization mechanism and timing on FashionMNIST
    ($N_{0} = 5\,000$, $C_4$ augmentation, $10\,000$ test samples), all with SGD.
    Results are mean $\pm$ std over 5 seeds.
    \textbf{Bold}: best in each column.
    \colorbox{gray!15}{Shaded}: best equivariance--accuracy trade-off.
    Full results across optimizers and trigger epochs are in \cref{tab:trigger-timing-full}.}
    \label{tab:symmetrization-comparison}
\end{table}
\paragraph{Results.} Our results are summarized in \cref{tab:symmetrization-comparison}.
First note  that none of our strategies yields perfect equivariance, as expected from the 
Monte Carlo residual quantified by \cref{thm:complexity-for-G-e-equivariance}, present even when $\eta \in H_G$. However, all strategies yield more equivariant models, and the orbit expansion model leads to a better overall performance.

We note that the performance across all metrics and strategies generally decreases when the intervention is applied later. This is natural, since a later trigger time exactly corresponds to a shorter training time post re-initalization. OSP and Sym.KL of the geometric average are the only exceptions to this trend, but this decrease is from much better values
compared to all other methods. 

The orbit expansion yields the best results. We hypothesize that the reason for this is that the orbit expansion results in parameters with additional symmetries compared to the orbit averaging techniques, due to the restricted number of base filters (cf.\ Figures~\ref{fig:symmetrization-panels-gcnn} and~\ref{fig:orbit-expansion-arrangements}).
These additional symmetries seem to induce more stability. To further explore such effects however remains future work.

In \cref{app:trigger-timing-ablation} we repeat the comparison with AdamW, whose adaptive rescaling violates the update commutation \eqref{eq:update_commute}; there an earlier trigger \emph{hurts} equivariance, so $H_{G}$ is less stable. This is further evidence for the correctness of our theory.

\section{Conclusion and limitations} In this paper, we theoretically analyzed the effect of data augmentation for variational inference of Bayesian neural networks. We showed that (under some weak conditions) when the underlying exponential family of variational distributions is closed under a symmetry and an invariant prior is used, a set $H_G$ of symmetric parameters corresponding to equivariant predictive maps becomes invariant under gradient training on augmented data. Based on this theorem, we proposed and tested several strategies to initialize the net in $H_G$. The orbit insertion strategy proved especially efficient, boosting both the performance and the equivariance compared to a baseline. 

Throughout our work, we have made a number of assumptions that limit the scope of our results. In particular, we focused on exponential families. Although this is a broad class of probability distributions, some notable distributions are not part of it, such as Gaussian mixtures. Furthermore, we have restricted our analysis (and correspondingly our experiments) to finite groups to enable a straightforward definition of the (finite) augmented dataset. An extension to continuous groups would require an augmentation approach based on sampling from the symmetry group. Although this is outside the scope of this work, we do not see a conceptual obstacle to extending our results in this way (see also the discussion in Appendix \ref{app:continuous-compact-groups}). A final limitation is the compatibility assumption~\eqref{eq:compatibility}. Given the flexibility of choice of intermediate representations, this restriction is however mild. 

\section*{Acknowledgments} This work was partially supported by the Wallenberg AI, Autonomous Systems and Software Program (WASP) funded by the Knut and Alice Wallenberg Foundation. The computations were enabled by resources provided by the National Academic Infrastructure for Supercomputing in Sweden (NAISS), partially funded by the Swedish Research Council through grant agreement no. 2025/22-1341.

\vspace*{-1ex}
\renewcommand*{\bibfont}{\normalfont\footnotesize}
\printbibliography

\ifbool{includeapp}{

\newpage
\appendix
\onecolumn

\section{From discrete to continuous compact groups}
\label[app]{app:continuous-compact-groups}

It is often the case in the geometric deep learning literature that virtually all results concerning finite groups can be generalized to a compact, continuous group -- simply by replacing the average over the group with an integral over the group with respect to the \emph{Haar measure} $\nu_G$
\begin{align*}
    \frac{1}{|G|}\sum_{g \in G} \quad \longrightarrow  \quad \int_G \mathrm{d}\nu_G(g).
\end{align*}
The Haar measure is the unique $G$-invariant probability measure on $G$ -- meaning that if $g\sim \nu_G$, also $hg\sim \nu_G$ for any $h\in G$. Note that for finite groups, the average over $G$ already is the Haar average.

In order for us to do this, there is one minor obstruction: the definition of the augmented likelihood. Note that for a finite group, this is well-defined:
\begin{align*}
    p(\mathcal{D}_{\mathrm{aug}} \, \vert \theta) = \prod_{(x_i,y_i)\in \mathcal{D}_{\mathrm{aug}}} p((x_i,y_i) \, \vert \theta) =\prod_{g\in G}\prod_{(x_i,y_i)\in \mathcal{D}}p((gx_i,gy_i) \, \vert \theta).
\end{align*}
For an infinite group, the product above (as well as the augmented dataset) becomes infinite, and hard to interpret. 

We can circumvent this by considering a sampled  augmentation -- we draw $S$ i.i.d. elements $g_j\in G$ (according to $\nu_G$) and form a randomly augmented dataset $\mathcal{D}_{S-\mathrm{aug}}=\{(g_jx_i,g_jy_i)\}_{i=1,j=1}^{N_0,S}$. The likelihood of this data set is
\begin{align*}
     p(\mathcal{D}_{S_\mathrm{aug}} \, \vert \theta) = \prod_{(x_i,y_i)\in \mathcal{D}_{S-\mathrm{aug}}} p((x_i,y_i) \, \vert \theta) =\prod_{j=1}^S\prod_{(x_i,y_i)\in \mathcal{D}}p((g_jx_i,g_jy_i) \, \vert \theta).
\end{align*}
Now note that we can maximize this by minimizing the negative log-likelihood, a rescaled version of which becomes
\begin{align*}
    -\frac{1}{S}\log p(\mathcal{D}_{\mathrm{aug}} \, \vert \theta) = - \frac{1}{S}\sum_{j=1}^S\sum_{(x_i,y_i)\in \mathcal{D}} \log p((gx_i,gy_i) \, \vert \theta)\,.
\end{align*}
As we increase the number of samples, this function converges towards
\begin{align*}
    -\int_G \sum_{(x_i,y_i)\in \mathcal{D}} \log p((gx_i,gy_i) \, \vert \theta) \mathrm{d}\nu_G(g)\,.
\end{align*}
If we define this function as $-\log p(\mathcal{D}_{\mathrm{aug}}\vert \theta)$ in the infinite case, all of our results go through, with proofs more or less verbatim.

It should be noted that in practice we will always need to work on a sampled augmentation space, so there will be a discretization error with respect to the group. Also note that in this formulation, we draw the sample once and then train on the resulting dataset -- we are not drawing new random group elements in each epoch, which often is done. To analyze this latter stochastic version requires other tools than we have used here, and we leave it to future work.

\section{Proof of \cref{prop:invariant-augmented-likelihood}}
\label[app]{app:proof-of-prop-invariant-augmented-likelihood}

\restateinvariantlikelihood*

\begin{proof}
The  compatibility of $\rho_{\Theta}$ \eqref{eq:compatibility} with the data representations implies that for any $g \in G$ and any input $x \in \mathcal{X}$,
    \begin{equation}
        f(gx;\theta) = g f(x;  \rho(g) ^{-1}\theta),
    \end{equation}
     implies
    \begin{align}
        p((gx,gy) \, \vert \, \theta) &= p(gy \, \vert \, gx, \theta)p(gx) = p(gy \, \vert \, f(gx;\theta))p(gx) \\
        &= p(gy \, \vert \, gf(x; \rho(g)^{-1}\theta))p(gx)
        = p(y \, \vert \, f(x; \rho(g)^{-1}\theta))p(gx),
    \end{align}
    or equivalently
    \begin{align} \label{eq:gshift}
        p((x,y) \, | \, \rho(g)\theta) = p(g^{-1}y \, \vert \, f(g^{-1}x;\theta))p(x),
    \end{align}

    We can now make the calculation
    \begin{align}
        L_{\mathrm{aug}}(\rho(g)\theta) &= \prod_{(x,y) \in \mathcal{D}_{\mathrm{aug}}} p((x, y) \mid \rho(g) \theta) = \prod_{(x,y) \in \mathcal{D}_{\mathrm{aug}}} p(g^{-1}y \, \vert \, f(g^{-1}x;\theta))p(x) \\
        &= \prod_{(x,y) \in \mathcal{D}_{\mathrm{aug}}}  p(g^{-1}y \, \vert \, f(g^{-1}x;\theta)) \cdot \prod_{(x,y)\in \mathcal{D}_{\mathrm{aug}}} p(x) \\
        &\stackrel{*}{=} \prod_{(x,y) \in \mathcal{D}_{\mathrm{aug}}}  p(y \, \vert \, f(x;\theta)) \cdot \prod_{(x,y)\in \mathcal{D}_{\mathrm{aug}}} p(x) \\
        &= \prod_{(x,y) \in \mathcal{D}_{\mathrm{aug}}}  p((x,y) \, \vert \, \theta) = L_{\mathrm{aug}}(\theta),
    \end{align}
    where (*) follows from the fact that $\mathcal{D}_{\mathrm{aug}}$ is closed under group transformations: $h\mathcal{D}_{\mathrm{aug}}=\mathcal{D}_{\mathrm{aug}}$ for any $h\in G$.
 \end{proof}

\section{Proof of \cref{thm:closure-under-push-forward}}
\label[app]{app:proof-of-thm-closure-under-push-forward}

\restateclosure*

\begin{proof}
    ``$\Leftarrow$'' This is a simple calculation \begin{align}
        (T_{g} \# q_{\eta})(\theta) &= h(\rho(g)^{-1}(\theta)) \exp(\eta^\top T(\rho(g)^{-1}(\theta)) - A(\eta)) \left|\det D(\rho(g)^{-1})(\theta)\right| \\
        &= h(\theta) \exp(b_{g}^\top T(\theta) + c_{g}) \exp(\eta^\top (M_{g} T(\theta) + d_{g}) - A(\eta)) \\
        &= h(\theta) \exp((M_{g}^\top \eta + b_{g})^\top T(\theta) - (A(\eta) - d_{g}^\top \eta - c_{g})) \\
        &=: h(\theta) \exp(\eta_{g}^\top T(\theta) - A_{g}(\eta))
    \end{align}
    It remains to verify that $A_{g}(\eta) = A(\eta) - d_{g}^\top\eta - c_{g}$ is indeed the log-partition function $A$ evaluated at $\eta_{g}$. 
    
    Using $\eta_{g} = M_{g}^\top\eta + b_{g}$, the affine condition $M_{g}T(\theta) = T(\rho(g)^{-1}\theta) - d_{g}$, and the base-measure condition $h(\theta)\exp(b_{g}^\top T(\theta)) = h(\rho(g)^{-1}\theta)\,|\det D(\rho(g)^{-1})(\theta)|\,e^{-c_{g}}$,
    \begin{align}
        A(\eta_{g})
        &= \log\int h(\theta)\exp\!\big(\eta^\top M_{g}T(\theta) + b_{g}^\top T(\theta)\big)\,d\theta \\
        &= \log\int h(\rho(g)^{-1}\theta)\,e^{-c_{g}}\,\big|\det D(\rho(g)^{-1})(\theta)\big|
           \exp\!\big(\eta^\top (T(\rho(g)^{-1}\theta) - d_{g})\big)\,d\theta \\
        &= -c_{g} - \eta^\top d_{g}
           + \log\int h(\rho(g)^{-1}\theta)\exp\!\big(\eta^\top T(\rho(g)^{-1}\theta)\big)
           \big|\det D(\rho(g)^{-1})(\theta)\big|\,d\theta \\
        &= -c_{g} - \eta^\top d_{g}
           + \log\int h(\theta')\exp\!\big(\eta^\top T(\theta')\big)\,d\theta'
           = A(\eta) - d_{g}^\top\eta - c_{g},
    \end{align}
    where the second-to-last step substitutes $\theta' = \rho(g)^{-1}\theta$, for which $|\det D(\rho(g)^{-1})(\theta)|\,d\theta = d\theta'$. Hence $A_{g}(\eta) = A(\eta_{g})$ and $\mathcal{T}_{g}\# q_{\eta} = q_{\eta_{g}}\in\mathcal{Q}$.

    ``$\Rightarrow$'' Suppose that there exists a mapping $\phi_{g}: H \to H$ such that for any $\eta \in H$ and $\theta \in \Theta$,
    \begin{equation}
        q_{\eta}(\rho(g)^{-1}(\theta)) \left|\det D(\rho(g)^{-1})(\theta)\right| = q_{\phi_{g}(\eta)}(\theta).
    \end{equation}
    
    First, we take the logarithm on both sides and use the form of $q_\eta$:
    \begin{align}
        &&&\log h(\rho(g)^{-1}(\theta)) + \eta^\top T(\rho(g)^{-1}(\theta)) - A(\eta) + \log \left|\det D(\rho(g)^{-1})(\theta)\right| \notag \\
        &&& \qquad =\phi_{g}(\eta)^\top T(\theta) - A(\phi_{g}(\eta)) + \log h(\theta) \\
        &\Rightarrow&& \eta^\top T(\rho(g)^{-1}(\theta)) + \log (h(\rho(g)^{-1}(\theta)) \left|\det D(\rho(g)^{-1})(\theta)\right|) + A(\phi_{g}(\eta)) - A(\eta) \notag \\
        &&& \qquad =\phi_{g}(\eta)^\top T(\theta) + \log h(\theta)\,. \label{equ:ast}
    \end{align}
    
    Take any $\eta_{1}, \eta_{2} \in H$, and subtract the equation \eqref{equ:ast} for $\eta_{2}$ from that for $\eta_{1}$:
    \begin{align}
        &(\eta_{1} - \eta_{2})^\top T(\rho(g)^{-1}(\theta)) - (\phi_{g}(\eta_{1}) - \phi_{g}(\eta_{2}))^\top T(\theta) \notag \\
        & \ =  A(\phi_{g}(\eta_{1})) - A(\eta_{1}) - A(\phi_{g}(\eta_{2})) + A(\eta_{2})\,.
    \end{align}

    Observe that the right-hand side is independent of $\theta$. Since the exponential family is minimal, $T(\theta)$ is linearly independent. To be precise, the components $\{1, T_{1}(\theta), \ldots, T_{k}(\theta)\}$ are linearly independent. Thus, both sides must be equal to some constant $C(\eta)$ independent of $\theta$. 

    Denote $u = \eta_{1} - \eta_{2}$ and $v = \phi_{g}(\eta_{1}) - \phi_{g}(\eta_{2})$. We have:
    \begin{equation}
        u^\top T(\rho(g)^{-1}(\theta)) - v^\top T(\theta) = C(\eta)\,.
    \end{equation}
    Since $\eta_{1}, \eta_{2}$ are arbitrary, $u$ can take any value in $\mathbb{R}^{k}$. Thus, we can choose $u$ to be each standard basis vector in $\mathbb{R}^{k}$. For each $j = 1, 2, \ldots, k$, we then get
    \begin{equation}
        T_{j}(\rho(g)^{-1}(\theta)) = \sum_{i=1}^{k} v_{i}^{(j)} T_{i}(\theta) + C_{j}\,.
    \end{equation}
    Denoting $M_{g} = [v^{(1)}, v^{(2)}, \ldots, v^{(k)}]^\top$ and $d_{g} = (C_{1}, C_{2}, \ldots, C_{k})^\top$, we conclude:
    \begin{equation}
        T(\rho(g)^{-1}(\theta)) = M_{g} T(\theta) + d_{g}\,.
    \end{equation}

    Next, we substitute this back into  \eqref{equ:ast}:
    \begin{align}
        &&&\eta^\top (M_{g} T(\theta) + d_{g}) + \log (h(\rho(g)^{-1}(\theta)) \left|\det D(\rho(g)^{-1})(\theta)\right|) - A(\eta) \notag \\
        &&& \qquad = \log h(\theta) + \phi_{g}(\eta)^\top T(\theta) - A(\phi_{g}(\eta)) \\
        &\Rightarrow&& (\eta^\top M_{g} - \phi_{g}(\eta)^\top) T(\theta) + \log (h(\rho(g)^{-1}(\theta)) \left|\det D(\rho(g)^{-1})(\theta)\right|) - \log h(\theta) + \eta^\top d_{g}  \notag \\
        &&& \qquad = A(\eta) - A(\phi_{g}(\eta))\,. \label{equ:astast}
    \end{align}

    Again, since the right-hand side is independent of $\theta$ and only depends on $\eta$, we take any two $\theta_{1}, \theta_{2} \in \Theta$ and subtract  \eqref{equ:astast} for $\theta_{2}$ from that for $\theta_{1}$ to obtain:
    \begin{equation}
        (\eta^\top M_{g} - \phi_{g}(\eta)^\top) (T(\theta_{1}) - T(\theta_{2})) + \log \frac{h(\rho(g)^{-1}(\theta_{1})) \left|\det D(\rho(g)^{-1})(\theta_{1})\right|}{h(\rho(g)^{-1}(\theta_{2})) \left|\det D(\rho(g)^{-1})(\theta_{2})\right|} - \log \frac{h(\theta_{1})}{h(\theta_{2})} = 0.
    \end{equation}
    Since the logarithmic terms do not depend on $\eta$, we can write this as
    \begin{equation}
        (\eta^\top M_{g} - \phi_{g}(\eta)^\top) (T(\theta_{1}) - T(\theta_{2})) = B(\theta_{1}, \theta_{2})\,,
    \end{equation}
    for some numbers $B(\theta_1,\theta_2)$. Now, again, since the family is minimal, the components $T(\theta)$ are linearly independent. This implies that we can choose a value for $\theta_1$ and $k$ different values for $\theta_2$, such that the matrix $(T(\theta_{1}) - T(\theta_{2}^\ell))_{\ell=1}^k\in \mathbb{R}^{k,k}$ becomes invertible. This yields a linear system of equations in the vector $\eta^\top M_{g} - \phi_{g}(\eta)^\top$, which has a unique solution which we denote by $-b_g(\theta)^{\top}$:
    \begin{equation}
        -b_{g}^\top(\theta) := \eta^\top M_{g} - \phi_{g}(\eta)^\top\,. \label{equ:triangle}
    \end{equation}
    Since the right hand side above is independent of $\theta$, the same must be true for $b_g$, so that we get
    \begin{equation}
        \phi_{g}(\eta) = M_{g}^\top \eta + b_{g}\,.
    \end{equation}

    Now plug \eqref{equ:triangle} back into  \eqref{equ:astast} to obtain
    \begin{equation}
        -b_{g}^\top T(\theta) + \log (h(\rho(g)^{-1}(\theta)) \left|\det D(\rho(g)^{-1})(\theta)\right|) - \log h(\theta) = A(\eta) - A(\phi_{g}(\eta)) - \eta^\top d_{g}\,.\label{eq:a-a-transf}
    \end{equation}

    Since the right-hand side is independent of $\theta$, the left-hand side must also be a constant independent of $\theta$. Since the left-hand side is independent of $\eta$, this constant is also independent of $\eta$. Denoting the constant by $c_{g}$, we have:
    \begin{align}
        &&\log (h(\rho(g)^{-1}(\theta)) \left|\det D(\rho(g)^{-1})(\theta)\right|) &= \log h(\theta) + b_{g}^\top T(\theta) + c_{g} \notag \\
        &\Longrightarrow& h(\rho(g)^{-1}(\theta)) \left|\det D(\rho(g)^{-1})(\theta)\right| &= h(\theta) \exp(b_{g}^\top T(\theta) + c_{g})
    \end{align}

    Finally, since the bracketed combination in \eqref{eq:a-a-transf} equals the constant $c_{g}$ established above, combining \eqref{eq:a-a-transf} with \eqref{equ:triangle} gives
    \begin{equation}
        A(\phi_{g}(\eta)) = A(\eta) - d_{g}^\top \eta - c_{g}.
    \end{equation}

    It remains to show that $g \mapsto \phi_{g}$ is an affine action of $G$ on $H$. Each $\phi_{g}(\eta) = M_{g}^\top\eta + b_{g}$ is affine, and $M_{g}$ is invertible: by minimality the rows $\{v^{(j)}\}$ chosen above are linearly independent, so $M_{g} = [v^{(1)},\dots,v^{(k)}]^\top$ is nonsingular.
    
    For the action property, the identity $e \in G$ has $\rho(e) = \mathrm{id}$, so
    \begin{equation}
        \mathcal{T}_{e} \# q_{\eta}(\theta) = q_{\eta}(\rho(e)^{-1}\theta) \left|\det D(\rho(e)^{-1})(\theta)\right| = q_{\eta}(\theta),
    \end{equation}
    giving $\phi_{e} = \mathrm{id}$. For $g_{1}, g_{2} \in G$, the chain rule yields
    \begin{align}
        \mathcal{T}_{g_{1} g_{2}} \# q_{\eta}(\theta)
        &= q_{\eta}\big(\rho(g_{2}^{-1})\rho(g_{1}^{-1})\theta\big)
           \big|\det D(\rho(g_{2}^{-1}))(\rho(g_{1}^{-1})\theta)\big|\,
           \big|\det D(\rho(g_{1}^{-1}))(\theta)\big| \notag \\
        &= q_{\phi_{g_{2}}(\eta)}\big(\rho(g_{1}^{-1})\theta\big)
           \big|\det D(\rho(g_{1}^{-1}))(\theta)\big|
        = q_{\phi_{g_{1}}(\phi_{g_{2}}(\eta))}(\theta),
    \end{align}
    Thus, we have $\phi_{g_{1} g_{2}}(\eta) = \phi_{g_{1}}(\phi_{g_{2}}(\eta))$. Since $\eta$ is arbitrary, we conclude that $\phi_{g_{1} g_{2}} = \phi_{g_{1}} \circ \phi_{g_{2}}$. Hence $g \mapsto \phi_{g}$ is an affine group action of $G$ on $H$, which completes the proof.
\end{proof}

\section{Proof of \cref{thm:equivariance-of-vi}}
\label[app]{app:proof-of-thm-equivariance-of-vi}

\restateequivariance*

\begin{proof}
    We prove the four claims in turn.
    \paragraph{(i) Loss invariance.} \cref{ass:invariant-prior}  implies that $q_{\eta_0}=q_{\phi_g(\eta_0)}$. Consequently, 
    \begin{align}
        \KL(q_{\phi_{g}(\eta)}(\theta) \Vert q_{\eta_{0}}(\theta)) &= \KL(q_{\phi_{g}(\eta)}(\theta) \Vert q_{\phi_{g}(\eta_{0})}(\theta)) \\ &= \KL(q_{\eta}(\rho(g)^{-1} \theta) \Vert  q_{\eta_{0}}(\rho(g)^{-1} \theta)) \\
        &= \int q_{\eta}(\rho(g)^{-1} \theta) \log \frac{q_{\eta}(\rho(g)^{-1} \theta)}{q_{\eta_{0}}(\rho(g)^{-1} \theta)} d(\rho(g)^{-1} \theta) \\
        &= \int q_{\eta}(\theta) \log \frac{q_{\eta}(\theta)}{q_{\eta_{0}}(\theta)} |\det D(\rho(g)^{-1})(\theta)| d\theta \\
        &= \int q_{\eta}(\theta) \log \frac{q_{\eta}(\theta)}{q_{\eta_{0}}(\theta)} d\theta = \KL(q_{\eta}(\theta) \Vert  q_{\eta_{0}}(\theta)),
    \end{align}
    where the third step is the transformation formula and the fourth uses the volume-preserving condition $|\det D(\rho(g)^{-1})(\theta)|=1$ (\cref{ass:volume-preserving}).

    We now use the invariance of the likelihood proven in  \cref{prop:invariant-augmented-likelihood} to get
    \begin{align}
        \mathbb{E}_{q_{\phi_{g}(\eta)}(\theta)}[\log p(\mathcal{D}_{\mathrm{aug}} \mid \theta)] &= \int q_{\phi_{g}(\eta)}(\theta) \log p(\mathcal{D}_{\mathrm{aug}} \mid \theta) d\theta \\
        &= \int q_{\eta}(\rho(g)^{-1} \theta) \log p(\mathcal{D}_{\mathrm{aug}} \mid \theta) d\theta \\
        &= \int q_{\eta}(\theta) \log p(\mathcal{D}_{\mathrm{aug}} \mid \rho(g)(\theta)) |\det D(\rho(g))(\theta)| d\theta \\
        &= \int q_{\eta}(\theta) \log p(\mathcal{D}_{\mathrm{aug}} \mid \theta) d\theta = \mathbb{E}_{q_{\eta}(\theta)}[\log p(\mathcal{D}_{\mathrm{aug}} \mid \theta)].
    \end{align}
    We again used the transformation formula and the volume-preserving conditions. Combining the two terms, we get
    \begin{equation}
        \mathcal{L}_{\beta}(\phi_{g}(\eta)) = \tfrac{1}{N} \mathbb{E}_{q_{\eta}(\theta)}[-\log p(\mathcal{D}_{\mathrm{aug}} \mid \theta)] + \beta \KL(q_{\eta}(\theta) \,\|\,  q_{\eta_{0}}(\theta)) = \mathcal{L}_{\beta}(\eta).
    \end{equation}

    \paragraph{(ii) Gradient equivariance and update commutativity.} Differentiating the identity $\mathcal{L}_{\beta}(\phi_{g}(\eta)) = \mathcal{L}_{\beta}(\eta)$ with respect to $\eta$ and using the chain rule together with $\phi_{g}(\eta) = M_{g}^\top\eta + b_{g}$ yields
    \begin{equation}
        M_{g} \nabla_{\eta} \mathcal{L}_{\beta}\big|_{\phi_{g}(\eta)} = \nabla_{\eta} \mathcal{L}_{\beta}(\eta),
    \end{equation}
    so that
    \begin{equation}
        \nabla_{\eta} \mathcal{L}_{\beta}(\phi_{g}(\eta)) = M_{g}^{-1} \nabla_{\eta} \mathcal{L}_{\beta}(\eta).
    \end{equation}
    When $M_{g}$ is orthogonal, $M_{g}^{-1} = M_{g}^\top$, we get
    \begin{align}
        \mathcal{U}(\phi_{g}(\eta)) &= \phi_{g}(\eta) - \alpha \nabla_{\eta} \mathcal{L}_{\beta}(\phi_{g}(\eta)) = M_{g}^\top \eta + b_{g} - \alpha M_{g}^{-1} \nabla_{\eta} \mathcal{L}_{\beta}(\eta) \\
        &= M_{g}^\top (\eta - \alpha \nabla_{\eta} \mathcal{L}_{\beta}(\eta)) + b_{g} = \phi_{g}(\mathcal{U}(\eta)).
    \end{align}

    \paragraph{(iii) Invariant subspace is preserved.} Let $\eta \in H_{G}$, i.e., $\phi_{g}(\eta) = \eta$ for all $g \in G$. By the update commutativity \eqref{eq:update_commute} from (ii),
    \begin{equation}
        \phi_{g}(\mathcal{U}(\eta)) = \mathcal{U}(\phi_{g}(\eta)) = \mathcal{U}(\eta)
        \quad \text{for all } g \in G,
    \end{equation}
    so $\mathcal{U}(\eta) \in H_{G}$. By induction on $t$, if $\eta^{(0)} \in H_{G}$ then $\eta^{(t)} \in H_{G}$ for all $t \geq 0$.

    \paragraph{(iv) Optimal solutions inherit symmetry.} Let $\eta^{\ast} \in H^{\ast} = \arg\min_{\eta} \mathcal{L}_{\beta}(\eta)$. By (i),
    \begin{equation}
        \mathcal{L}_{\beta}(\phi_{g}(\eta^{\ast})) = \mathcal{L}_{\beta}(\eta^{\ast}) = \min_{\eta} \mathcal{L}_{\beta}(\eta),
    \end{equation}
    so $\phi_{g}(\eta^{\ast}) \in H^{\ast}$ for all $g \in G$; i.e., $H^{\ast}$ is closed under $\phi_{g}$. If the minimizer is unique, then $\phi_{g}(\eta^{\ast}) = \eta^{\ast}$ for all $g \in G$, meaning $\eta^{\ast} \in H_{G}$.
\end{proof}

\section{Proof of \cref{thm:prior-independent-convergence}}
\label[app]{app:proof-of-thm-prior-independent-convergence}

\restateconvergence*

\begin{proof}
    Note that the log-likelihood does not depend on the prior. Hence, using the same calculation as in the proof of Theorem \ref{thm:equivariance-of-vi} (i),
    we show that $R_{\mathrm{aug}}$ is $G$-invariant on the natural parameter space:
    \begin{equation}
        R_{\mathrm{aug}}(\phi_{g}(\eta)) = R_{\mathrm{aug}}(\eta), \quad \forall g \in G,\, \eta \in H.
    \end{equation}
    In particular, the set $H^{\ast}_{R} := \arg\min_{\eta} R_{\mathrm{aug}}(\eta)$ is closed under $\phi_{g}$ for all $g \in G$.

    Fix any $\eta^{\ast}_{R} \in H^{\ast}_{R}$. By the optimality of $\eta^{\ast}_{\beta}(p_{0})$ for $\mathcal{L}_{\beta}(\cdot; p_{0})$,
    \begin{equation}
        R_{\mathrm{aug}}(\eta^{\ast}_{\beta}(p_{0})) + \beta \KL(q_{\eta^{\ast}_{\beta}(p_{0})} \Vert p_{0}) \leq R_{\mathrm{aug}}(\eta^{\ast}_{R}) + \beta \KL(q_{\eta^{\ast}_{R}} \Vert p_{0}).
    \end{equation}
    Dropping the non-negative KL term on the left and using $R_{\mathrm{aug}}(\eta^{\ast}_{R}) = R^{\ast}_{\mathrm{aug}}$,
    \begin{equation}
        R_{\mathrm{aug}}(\eta^{\ast}_{\beta}(p_{0})) \leq R^{\ast}_{\mathrm{aug}} + \beta \KL(q_{\eta^{\ast}_{R}} \Vert p_{0}).
    \end{equation}
    Taking the infimum over $\eta^{\ast}_{R} \in H^{\ast}_{R}$ on the right yields the stated bound
    \begin{equation}
        R_{\mathrm{aug}}(\eta^{\ast}_{\beta}(p_{0})) \leq R^{\ast}_{\mathrm{aug}} + \beta \cdot C(p_{0}).
    \end{equation}
    Combined with $R_{\mathrm{aug}}(\eta^{\ast}_{\beta}(p_{0})) \geq R^{\ast}_{\mathrm{aug}}$ (by the definition of $R^{\ast}_{\mathrm{aug}}$), this gives
    \begin{equation}
        R^{\ast}_{\mathrm{aug}} \leq R_{\mathrm{aug}}(\eta^{\ast}_{\beta}(p_{0})) \leq R^{\ast}_{\mathrm{aug}} + \beta \cdot C(p_{0}).
    \end{equation}
    For any fixed $p_{0}$ with $C(p_{0}) < \infty$, sending $\beta \to 0$ (equivalently $N \to \infty$, since $\beta = 1/N$) yields $R_{\mathrm{aug}}(\eta^{\ast}_{\beta}(p_{0})) \to R^{\ast}_{\mathrm{aug}}$, independent of $p_{0}$.
\end{proof}

\section{Proof of \cref{thm:complexity-for-G-e-equivariance}}
\label[app]{app:proof-of-thm-complexity-for-G-e-equivariance}

Before stating the proof, we recall the definitions of the empirical and Monte Carlo equivariance defects.

\begin{definition}[Empirical and Monte Carlo equivariance defects]
    Given a dataset $\{x_{i}\}_{i=1}^{N_{0}}$, the empirical equivariance defect is
    \begin{equation}
    \label{equ:def-delta-tilde}
        \widetilde{\Delta}^{\mathrm{eq}}_{F}(\eta) := \frac{1}{N_{0}}\sum_{i=1}^{N_{0}} \frac{1}{|G|} \sum_{g \in G} \left\|F_{\eta}(gx_{i}) - g F_{\eta}(x_{i})\right\|^{2}.
    \end{equation}
    Its Monte Carlo approximation using $T$ posterior samples is
    \begin{equation}
    \label{equ:def-delta-hat}
        \widehat{\Delta}^{\mathrm{eq}}_{F}(\eta) := \frac{1}{N_{0}}\sum_{i=1}^{N_{0}} \frac{1}{|G|} \sum_{g \in G} \left\|\widehat{F}_{\eta}(gx_{i}) - g \widehat{F}_{\eta}(x_{i})\right\|^{2}, \quad \widehat{F}_{\eta}(x) := \frac{1}{T} \sum_{t=1}^{T}  f(x; \theta^{(t)}),
    \end{equation}
    where $\{\theta^{(t)}\}_{t=1}^{T}$ are i.i.d.\ samples from $q_{\eta}(\theta)$.
\end{definition}

We now restate \cref{thm:complexity-for-G-e-equivariance} and prove it thereafter.

\restatecomplexity*

\begin{proof}

Suppose that there exists $M > 0$ such that $\|f(x;\theta)\| \le M$ for all $x \in \mathcal{X}$ and $q_{\eta}$-a.s.\ $\theta$. Let us begin by noting that a simple application of Jensen's inequality shows that for each $x$,
\begin{align}
 \|F_{\eta}(x)\|& = \|\mathbb{E}_{q_{\eta}}[f(x;\theta)]\| \le \mathbb{E}_{q_{\eta}}[\|f(x;\theta)\|] \le M, \\
 \|\widehat{F}_{\eta}(x)\| &= \frac{1}{T}\left\|  \sum_{t=1}^T f(x;\theta^{(t)})\right\| \leq \frac{1}{T} \sum_{t=1}^T \left\|f(x; \theta^{(t)})\right\| \le M.
\end{align}
 Using the triangle inequality and orthogonality of the group action, this yields
 \begin{align}
     \|F_{\eta}(gx)\pm gF_\eta(x)\| \leq 2M, \quad \|\widehat{F}_\eta(gx)\pm g\widehat{F}_\eta(x)\| \leq 2M \label{eq:bounds}
 \end{align}
Hence, the quantities we need to estimate are essentially sums of i.i.d. bounded terms. This makes standard concentration inequalities viable. We now carry out the details.
    
    \paragraph{Step 1: Dataset-level concentration.} Define
    \begin{equation}
        Z_{i} := \frac{1}{|G|} \sum_{g \in G} \|F_{\eta}(gx_{i}) - g F_{\eta}(x_{i})\|^{2}.
    \end{equation}
    These terms are bounded, $Z_{i} \in [0,4M^{2}]$, their common expected value is $\mathbb{E}[Z_{i}] = \Delta^{\mathrm{eq}}_{F}(\eta)$, and their mean is $\frac{1}{N_{0}}\sum_{i=1}^{N_{0}} Z_{i} = \widetilde{\Delta}^{\mathrm{eq}}_{F}(\eta)$. Applying Hoeffding's inequality, we get for any $ t > 0$,
    \begin{equation}
        \mathbb{P}\!\left(\widetilde{\Delta}^{\mathrm{eq}}_{F}(\eta) -\Delta^{\mathrm{eq}}_{F}(\eta) \ge  t\right) \le \exp\!\left(-\frac{2 N_{0} t^{2}}{(4M^{2})^{2}}\right) = \exp\!\left(-\frac{N_{0}  t^{2}}{8 M^{4}}\right).
    \end{equation}
    Setting the right-hand side to $\delta/2$ and solving for $ t$, we get that with probability at least $1-\delta/2$,
    \begin{equation} \label{eq:hoeffding}
        \widetilde{\Delta}^{\mathrm{eq}}_{F}(\eta)  \leq \Delta^{\mathrm{eq}}_{F}(\eta)+ C_{\mathrm{data}}\sqrt{\frac{\log(2/\delta)}{N_{0}}},
    \end{equation}
    where $C_{\mathrm{data}} := 2\sqrt{2}\,M^{2}$.
    
    \paragraph{Step 2: Monte Carlo concentration.} The estimator $\widehat{\Delta}^{\mathrm{eq}}_{F}(\eta)$ differs from $\widetilde{\Delta}^{\mathrm{eq}}_{F}(\eta)$ only through replacing the posterior expectation $F_{\eta}(x)$ by the empirical average $\widehat{F}_{\eta}(x)$. Note that in contrast to above, $\mathbb{E}(\widehat{\Delta}_F^{\mathrm{eq}})$ is not equal to $\widetilde{\Delta}_F^{\mathrm{eq}}(\eta)$. There instead is a small bias term. We first control that bias, and then  bound $\widehat{\Delta}^{\mathrm{eq}}_{F}(\eta)$ around its mean via McDiarmid's inequality.

    \emph{Bias.} Writing $D := F_{\eta}(gx) - g F_{\eta}(x)$ and $\widehat{D} := \widehat{F}_{\eta}(gx) - g \widehat{F}_{\eta}(x)$, we have
    \begin{equation}
        \mathbb{E}\!\left[\|\widehat{D}\|^{2}\right] - \|D\|^{2} = \mathrm{tr}\!\left(\mathrm{Cov}(\widehat{D})\right) \le \frac{4 M^{2}}{T},
    \end{equation}
    since $\widehat{D}$ is an average of $T$ i.i.d.\ terms each bounded in norm by $2M$. Averaging over $G$ and the data samples yields
    \begin{equation} \label{eq:bias}
        \mathbb{E}[\widehat{\Delta}^{\mathrm{eq}}_{F}(\eta)] - \widetilde{\Delta}^{\mathrm{eq}}_{F}(\eta) \le \frac{C_{\mathrm{bias}}}{T}\leq \frac{C_{\mathrm{bias}}}{\sqrt{T}},
    \end{equation}
    with $C_{\mathrm{bias}} = 8 M^{2}$. 
    
    \emph{Bounded differences.} Replacing a single sample $\theta^{(s)}$ by $\theta'^{(s)}$ changes any $\widehat{F}_{\eta}(x)$ by at most $\frac{1}{T}\|f(x;\theta^{(s)}) - f(x;\theta'^{(s)})\| \le 2M/T$. By the triangle inequality, the same replacement changes every $\widehat{F}_{\eta}(gx_{i}) - g \widehat{F}_{\eta}(x_{i})$ less than $4M/T$. Using the identity $\|a\|^{2} - \|b\|^{2} = \langle a-b, a+b\rangle$ together with bounds $\|a\| = \|\widehat{F}_\eta(gx_{i}) - g\widehat{F}_\eta(x_{i})\| \le 2M,\, \|b\| = \|\widehat{F}'_\eta(gx_{i}) - g\widehat{F}'_\eta(x_{i})\| \le 2M$ from~\eqref{eq:bounds}, we get
    \begin{equation}
        \left|\|\widehat{F}_{\eta}(gx_{i}) - g \widehat{F}_{\eta}(x_{i})\|^{2} - \|\widehat{F}'_{\eta}(gx_{i}) - g \widehat{F}'_{\eta}(x_{i})\|^{2}\right| \le \frac{4M}{T} \cdot 4M = \frac{16 M^{2}}{T}.
    \end{equation}
    Averaging over $i$ and $g$ preserves this bound, so $\widehat{\Delta}^{\mathrm{eq}}_{F}(\eta)$ satisfies the bounded-differences condition with constants $c_{s} = 16 M^{2}/T$ for each $s = 1, \ldots, T$. 
    
    McDiarmid's inequality \citep{McDiarmid_1989,doobREGULARITYPROPERTIESCERTAIN} states that any function $\Phi$ of independent random variables satisfies
    \begin{equation}
    \mathbb{P}\big(\mathbb{E}[\Phi] - \Phi \ge t\big) \le \exp\!\left(-\frac{2 t^{2}}{\sum_{s=1}^{T} c_{s}^{2}}\right),
    \end{equation}
    where $c_{s}$ is the bounded-difference constant of $\Phi$ in its $s$-th argument, i.e. substituting the $s$-th argument of $\Phi$ with a different value changes $\Phi$ by at most $c_s$.
    
    Substituting $\Phi = \widehat{\Delta}^{\mathrm{eq}}_{F}(\eta)$ and $c_{s} = 16 M^{2}/T$, we have $\sum_{s=1}^{T} c_{s}^{2} = T \cdot (16 M^{2}/T)^{2} = 256 M^{4}/T$. Hence,
    \begin{equation}
        \mathbb{P}\!\left(\mathbb{E}[\widehat{\Delta}^{\mathrm{eq}}_{F}(\eta)] - \widehat{\Delta}^{\mathrm{eq}}_{F}(\eta) \ge t\right) \le \exp \left(-\frac{T t^{2}}{128 M^{4}}\right).
    \end{equation}
    
    Setting the right-hand side to $\delta/2$ and solving for $t$ yields, with probability at least $1-\delta/2$,
    \begin{equation} \label{eq:McDiarmid}
        \mathbb{E}[\widehat{\Delta}^{\mathrm{eq}}_{F}(\eta)] - \widehat{\Delta}^{\mathrm{eq}}_{F}(\eta) \le \frac{16 M^{2}}{\sqrt{2T}}\sqrt{\log(2/\delta)} = C_{\mathrm{mc}}\sqrt{\frac{\log(2/\delta)}{T}},
    \end{equation}
    where $C_{\mathrm{mc}} = 16 M^{2}/\sqrt{2} = 8\sqrt{2}\,M^{2}$.

    \paragraph{Step 3: Combining deviation and bias.} We now combine the above to get
    \begin{align}
        \widehat{\Delta}_F^{\mathrm{eq}}(\eta) &= \Delta^{\mathrm{eq}}_F(\eta) + \underbrace{(\widetilde{\Delta}_F^{\mathrm{eq}}(\eta) - \Delta^{\mathrm{eq}}_F(\eta))}_{\eqref{eq:hoeffding}} + \underbrace{(\mathbb{E}[\widehat{\Delta}_F^{\mathrm{eq}}(\eta)] - \widetilde{\Delta}_F^{\mathrm{eq}}(\eta))}_{\eqref{eq:bias}} + \underbrace{(\widehat{\Delta}_F^{\mathrm{eq}}(\eta) - \mathbb{E}[\widehat{\Delta}_F^{\mathrm{eq}}(\eta)] )}_{\eqref{eq:McDiarmid}} \\
        \leq &\Delta^{\mathrm{eq}}_F(\eta) +C_{\mathrm{data}}\sqrt{\frac{\log(2/\delta)}{N_0}}+C_{\mathrm{bias}}\sqrt{\frac{1}{T}}+C_{\mathrm{mc}}\cdot \sqrt{\frac{\log(2/\delta)}{T}},
    \end{align}
    which yields the first claim with $C_{\mathrm{weight}}= C_{\mathrm{mc}} + C_{\mathrm{bias}}\cdot\sqrt{1/\log(2)}$.

    \paragraph{Step 4: Sample requirements.} For the decomposition $\varepsilon = \varepsilon_{0} + \varepsilon_{\mathrm{data}} + \varepsilon_{\mathrm{weight}} + \varepsilon_{\mathrm{bias}}$, requiring each term in the bound to be at most its allotted budget gives
    \begin{equation}
        C_{\mathrm{data}}\sqrt{\frac{\log(2/\delta)}{N_{0}}} \leq \varepsilon_{\mathrm{data}}, 
        \quad 
        C_{\mathrm{weight}}\sqrt{\frac{\log(2/\delta)}{T}} \leq \varepsilon_{\mathrm{weight}}.
    \end{equation}
    Squaring and solving for $N_{0}$ and $T$ yields the stated requirements.
\end{proof}

\section{Geometric averaging stays in $\mathcal{Q}$}
\label[app]{app:geometric-averaging-stays-in-Q}

\begin{lemma}[Geometric mean invariance and membership]
\label{lmm:geometric-mean-invariance-and-membership}
    Suppose $\mathcal{Q}$ is closed under all push-forwards $\mathcal{T}_{g}$ as in \cref{thm:closure-under-push-forward}, i.e., for each $g \in G$ and any $\eta \in H$,
    \begin{equation}
        (\mathcal{T}_{g} \# q_{\eta})(\theta) = h(\theta) \exp\big(\eta_{g}^\top T(\theta) - A_{g}(\eta)\big),
    \end{equation}
    for some $\eta_{g} \in H$. Then the geometric mean
    \begin{equation}
        \tilde{q}(\theta) \propto \left[\prod_{g \in G} \Big(\mathcal{T}_{g} \# q_{\eta}\Big)(\theta)\right]^{\frac{1}{|G|}}
    \end{equation}
    is invariant under the group action and can be written as an exponential-family member with the same base measure $h$ and sufficient statistic $T$:
    \begin{equation}
        \tilde{q}(\theta) = h(\theta) \exp\Big(\tilde{\eta}^\top T(\theta) - \tilde{A}\Big),
    \end{equation}
    where
    \begin{equation}
        \tilde{\eta} = \frac{1}{|G|} \sum_{g \in G} \eta_{g}, \quad \tilde{A} = \frac{1}{|G|} \sum_{g \in G} A(\eta_{g}).
    \end{equation}
\end{lemma}

\begin{proof}
    Using closure under push-forward, for each $g$,
    \begin{equation}
        (\mathcal{T}_{g} \# q_{\eta})(\theta) = h(\theta) \exp \big( \eta_{g}^\top T(\theta) - A(\eta_{g}) \big).
    \end{equation}
    Taking the $\frac{1}{|G|}$-power and multiplying over all $g \in G$ yields
    \begin{equation}
        \tilde{q}(\theta) = h(\theta) \exp\Big(\frac{1}{|G|} \sum_{g \in G} \eta_{g}^\top T(\theta) - \frac{1}{|G|} \sum_{g \in G}  A(\eta_{g})\Big),
    \end{equation}
    which is of the standard exponential-family form with the stated $\tilde{\eta}$ and $\tilde{A}$. Invariance follows because for any $g_{0} \in G$, the set $\{\mathcal{T}_{g} \# q\}_{g\in G}$ is permuted under $\mathcal{T}_{g_{0}}$ and the product over $G$ (with equal exponents) is unchanged. Hence $\tilde{q}$ is invariant and belongs to $\mathcal{Q}$.
\end{proof}

\section{The two orbit averaging strategies}
\label[app]{app:symmetrization-strats}
In this section, we show how geometric averaging and projection naturally arise as two instances of orbit averaging. First, let us present in detail the canonical strategy of defining parameter space representations discussed in the main paper. Remember that a classical neural network (e.g. an MLP) consists of layers $\theta_\ell$ of affine maps, whose domains and co-domains are vector spaces, say $\theta_\ell:~\gV_\ell~\to~\gV_{\ell+1}$, with $\gV_0=\gX$ and $\gV_L = \gY$. One now introduces representations $\rho_\ell$ on each $\gV_\ell$, with $\rho_{0}=\rho_{\gX}$ and $\rho_L =\rho_{\gY}$ to keep the construction compatible with the given in- and output representations, and defines
\begin{align}
    \rho_{\ell}(g)(\theta_\ell) = \rho_{\gV_{\ell+1}}(g) \circ \theta_\ell \circ \rho_{\gV_\ell}(g)^{-1} \label{eq:liftedrep}
\end{align}
If the activation functions are equivariant with respect to the intermediate representation, the above definition always yields in a representation satisfying \eqref{eq:compatibility} -- see \citet{nordenfors2025optimization}. 

For the networks that we are considering in Section \ref{sec:symmetrization-of-the-variational-posterior}, the input and intermediate spaces consist of \emph{channels of images}. That is, for a space $\mathcal{I}$ of images, we set
\begin{align}
    \gV_\ell = \mathcal{I}^{n_\ell}.
\end{align}
On the space $\mathcal{I}$, $C_4$ acts through a rotation operator $R:\mathcal{I}\to \mathcal{I}$. We can extend this to $\gV_\ell$ in several ways. First, we can rotate channel-wise:
\begin{align}
    (\rho_{\gV_\ell}^{\mathrm{ch}} (i) v)_k = R^{i} v_k.
\end{align}
Alternatively, we can additionally act via 'rolling' along the channel dimension. More concretely, one divides the channels into groups of four, and then defines
\begin{align}
    (\rho_{\gV_\ell}^{\mathrm{gcnn}}(i) v)_{4n + k} = R^iv_{4n +(k+i \, \mathrm{mod} \, 4)}.
\end{align}
The second representation is the one that leads to GCNNs.

Let us now apply \eqref{eq:liftedrep} to a convolutional operator,
\begin{align}
    (\theta_\ell v)_k = \sum_j \eta^{kj}*v_j.
\end{align}
Note that rotation in the image space and convolutions are compatible, in the following sense:
\begin{align}
    (R\eta)*(Rv) = R(\eta*v),
\end{align}

\paragraph{Geometric average.} For the channel-wise representation, we obtain:
\begin{align}
    (\rho^{\mathrm{ch}}_{\ell}(i)v)_k &= \rho_{\gV_{\ell+1}}^{\mathrm{ch}}(i) ([\theta_\ell \circ \rho_{\gV_\ell}^{\mathrm{ch}}(i)^{-1}]v)_k = R^i\sum_{j}\eta^{kj}*(\rho_{\gV_\ell}^{\mathrm{ch}}(i)^{-1}v)_j \\
    &= R^i\sum_{j}\eta^{kj}*(R^{-i}v_j) = \sum_{j}(R^i\eta^{kj})*v_j. 
\end{align}
That is, $\rho^{\mathrm{ch}}$ acts via rotating each filter individually: $(\eta^{kj})_{k,j} \to (R^i\eta^{kj})_{k,j}$. Correspondingly, the orbit average averages all four rotations of each filter, yielding to a filter-wise symmetrization.

\paragraph{Projection.} For the GCNN-representation, let us restrict the calculation to the filters between two four-blocks, i.e. $n=0$. We get
\begin{align}
      (\rho^{\mathrm{gcnn}}_{\ell}(i)v)_{k} &= \rho_{\gV_{\ell+1}}^{\mathrm{gcnn}}(i) ([\theta_\ell \circ \rho_{\gV_\ell}^{\mathrm{gcnn}}(i)^{-1}]v)_k = R^i\sum_{j}\eta^{(k+i \, \mathrm{mod} \, 4), j}*(\rho_{\gV_\ell}^{\mathrm{gcnn}}(i)^{-1}v)_{j} \\
    &= R^i\sum_{j}\eta^{(k+i \, \mathrm{mod} \, 4),j}*(R^{-i}v_{j-i \, \mathrm{mod}\, 4}) = \sum_{j}R^i\eta^{(k+i \, \mathrm{mod} \, 4),(j+i \, \mathrm{mod} \, 4)}*v_j.
\end{align}
Thus, $\rho^{\mathrm{gcnn}}_{\ell}$ acts through both rotating and rolling the filters along the diagonals of each four-by-four block. The resulting average operation will calculate an average of rotated filters along each such diagonal, and then insert it into the block in rotated versions, giving $4\times4$-blocks as in Figure \ref{fig:proj-vs-avg-trajectory}.

\section{Expansion along orbit: filter arrangement ablation}
\label[app]{app:expansion-along-orbit-filter-arrangement-ablation}

Let us now discuss the orbit expansion strategy. The idea is to construct a set of equivariant filters, starting from a smaller set of generic filters. To keep the exposition light, we start with a single filter $\eta$ and expand it to a single four-times-four block -- the construction naturally generalizes to filter banks.

\begin{figure}[t]
    \centering
    \includegraphics[width=1\linewidth, trim=0.29cm 0.9cm 0.8cm 1cm, clip]{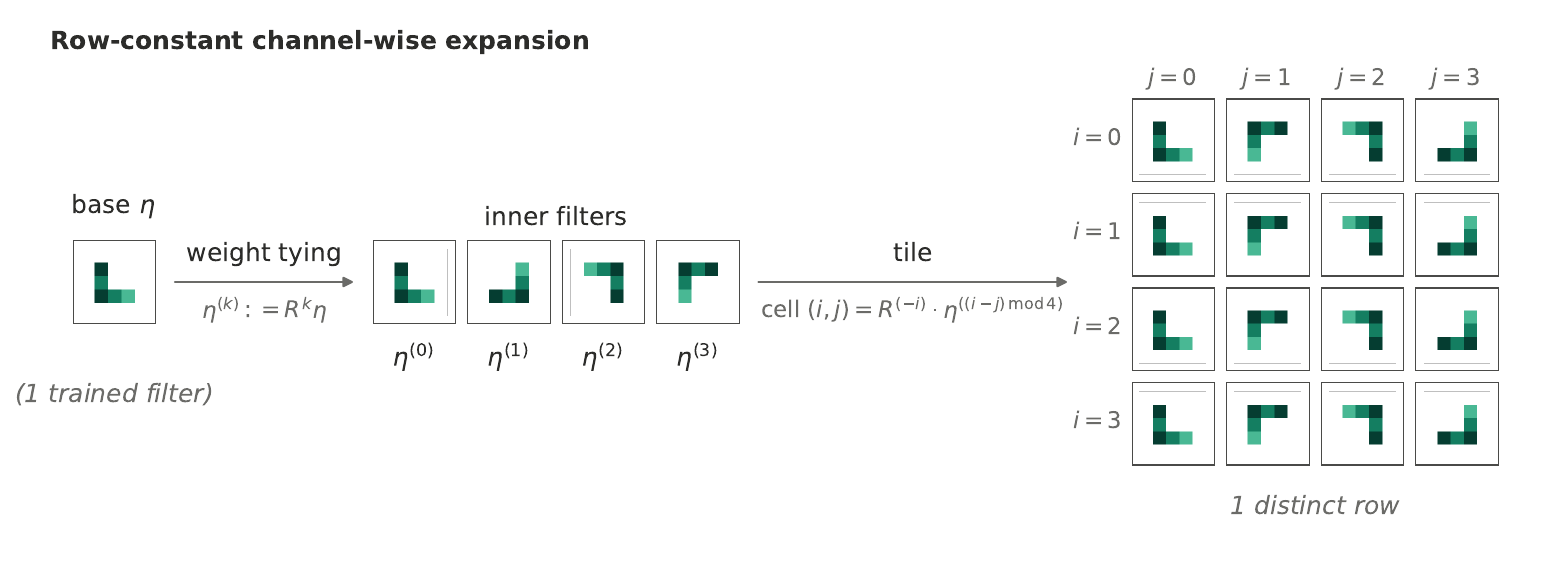}
    \includegraphics[width=1\linewidth, trim=0.29cm 0.9cm 0.8cm 1cm, clip]{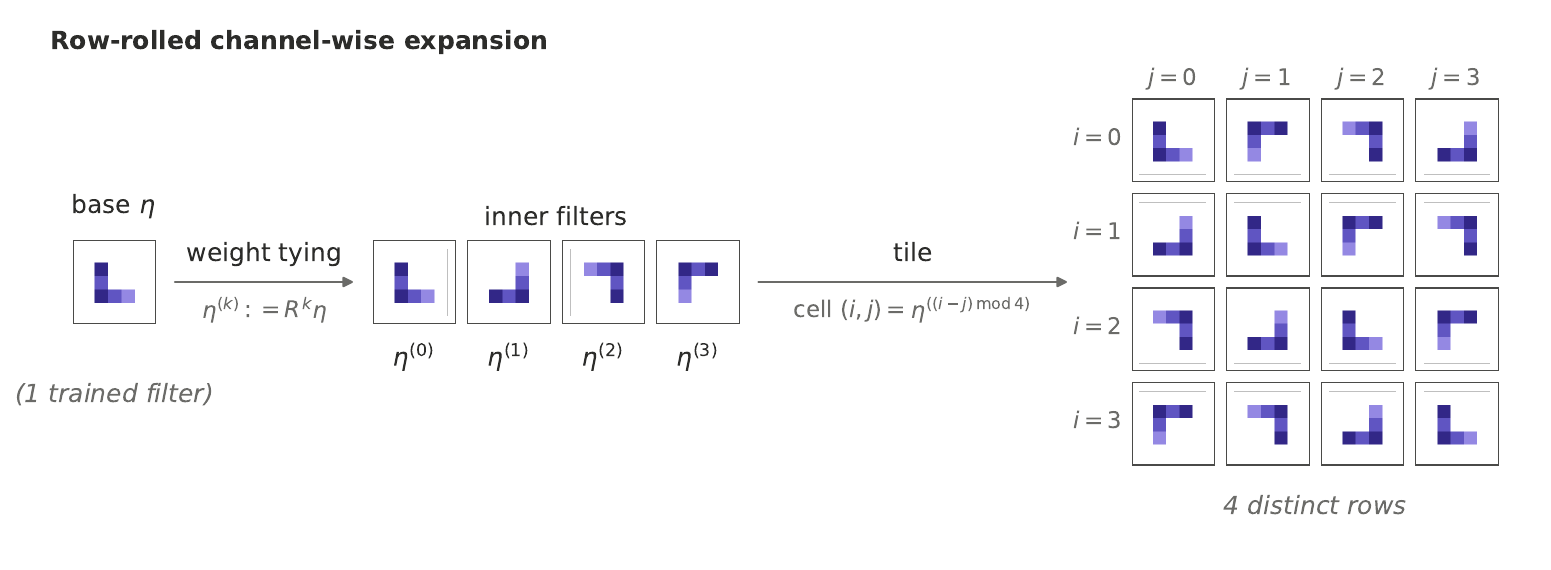}
    \includegraphics[width=1\linewidth, trim=0.29cm 0.9cm 0.8cm 1cm, clip]{plots/symmetrization_panels_ds.pdf}
    \caption{Three filter arrangements for orbit expansion under $C_{4}$,
    starting from a single base filter $\eta$ trained in the
    width-$1/|G|$ Stage 1 network.  All three share the same first row, and only differ in how that row is expanded to the other rows.
    \textbf{Top:} Row-constant channel-wise expansion
    \textbf{Middle:} Row-rolled channel-wise expansion
    \textbf{Bottom:} GCNN-expansion }\label{fig:orbit-expansion-arrangements}
\end{figure}

Since there are several notions of equivariance, there are different ways to achieve the equivariance of the larger filterbank. In Appendix~\ref{app:symmetrization-strats}, we have already derived conditions for equivariance when both the input and output representations are equal to $\rho^{\mathrm{ch}}$ and $\rho^{\mathrm{gcnn}}$, respectively. What are the conditions when the representations are unequal? Let us say that the representation on the layer $\gV_1$ is $\rho^{\mathrm{gcnn}}$ and the one on $\gV_2$ is $\rho^{\mathrm{ch}}$. The lifted representation then becomes
  \begin{align*}
           (\rho^{}_{1}(i)v)_k &= \rho_{\gV_{2}}^{\mathrm{ch}}(i) ([\theta_\ell \circ \rho_{\gV_1}^{\mathrm{gcnn}}(i)^{-1}]v)_k = R^i\sum_{j}\eta^{kj}*(\rho_{\gV_1}^{\mathrm{gcnn}}(i)^{-1}v)_j \\
    &= R^i\sum_{j}\eta^{kj}*(R^{-i}v_{j-i}) = \sum_{j}(R^i\eta^{k,j+i})*v_j. 
    \end{align*}
    Consequently, a layer is equivariant if and only if 
    \begin{equation}
        R^i\eta^{k,j+i}=\eta^{k,j} \text{ for all }i,j,k.
    \end{equation} 
    This means that within each row,  the filters should be rotations of each other. In Figure \ref{fig:orbit-expansion-arrangements}, we plot three ways to construct filters that fulfill this condition. We take the filter $\eta$ and rotate it to form the first row, and expand that structure into an entire block in three different ways
    \begin{itemize}
        \item In the \textbf{row-constant channel-wise expansion}, we simply copy the first row to the others.
        \item In the \textbf{row-rolled channel-wise expansion}, we roll the filters along the row dimension to define the new rows. 
        \item In the \textbf{GCNN expansion}, we roll the filters along the row dimension and rotate them to define the new rows. 
    \end{itemize}  
    All these construction yield per se valid structures of equivariance, since there is a freedom to choose the representation on the second feature space.  At the same time, only the third construction is equivariant in the '$\rho^{\mathrm{gcnn}}$-to-$\rho^{\mathrm{gcnn}}$ sense', so it is special. See Figure~\ref{fig: equivarization} for a summary.

    \begin{remark}
        One could also let $\rho^{ch}$ act on the space $\gV_1$, but this would necessitate using rotation-symmetric filters in the first layer, hurting expressiveness severely. We therefore do not consider that construction.
    \end{remark}

    \begin{table}[t]
\centering
\small
\setlength{\tabcolsep}{4pt}
\begin{tabular}{@{}ll ccc@{}}
\toprule
\textbf{Variant} & \textbf{Epoch}
  & \textbf{Acc.\,(\%)} $\uparrow$
  & \textbf{OSP} $\uparrow$
  & \textbf{Sym.\,KL} $\downarrow$ \\
\midrule
\multirow{3}{*}{\shortstack[l]{Row-constant channel-wise \\expansion}}
  & 0   & $80.2{\scriptstyle\pm0.3}$ & $0.957{\scriptstyle\pm.004}$ & $0.0078{\scriptstyle\pm.0011}$ \\
  & 20  & $80.3{\scriptstyle\pm0.5}$ & $0.956{\scriptstyle\pm.004}$ & $0.0084{\scriptstyle\pm.0014}$ \\
  & 100 & $80.1{\scriptstyle\pm0.5}$ & $0.955{\scriptstyle\pm.003}$ & $0.0086{\scriptstyle\pm.0006}$ \\
\midrule
\multirow{3}{*}{\shortstack[l]{Row-rolled channel-wise \\expansion}}& 0   & $80.2{\scriptstyle\pm0.3}$ & $0.956{\scriptstyle\pm.004}$ & $0.0084{\scriptstyle\pm.0010}$ \\
  & 20  & $80.3{\scriptstyle\pm0.5}$ & $0.956{\scriptstyle\pm.004}$ & $0.0088{\scriptstyle\pm.0013}$ \\
  & 100 & $80.1{\scriptstyle\pm0.5}$ & $0.954{\scriptstyle\pm.003}$ & $0.0086{\scriptstyle\pm.0006}$ \\
\midrule
\multirow{3}{*}{GCNN-expansion}& 0   & $\mathbf{80.4}{\scriptstyle\pm0.4}$ & $0.963{\scriptstyle\pm.003}$ & $0.0057{\scriptstyle\pm.0004}$ \\
  & 20  & $80.4{\scriptstyle\pm0.6}$ & $0.966{\scriptstyle\pm.004}$ & $0.0048{\scriptstyle\pm.0009}$ \\
  & \cellcolor{gray!15}100 & \cellcolor{gray!15}$80.0{\scriptstyle\pm0.5}$ & \cellcolor{gray!15}$\mathbf{0.968}{\scriptstyle\pm.003}$ & \cellcolor{gray!15}$\mathbf{0.0042}{\scriptstyle\pm.0007}$ \\
\bottomrule
\end{tabular}
\caption{Effect of filter arrangement and trigger timing in orbit
expansion (FashionMNIST, $N_{0} = 5\,000$, $C_4$ augmentation, SGD,
5 seeds), evaluated at the best-accuracy checkpoint.
\textbf{Bold}: best in each column.
\colorbox{gray!15}{Shaded}: best equivariance-accuracy trade-off.}
\label{tab:orbit-expansion-ablation}
\end{table}

    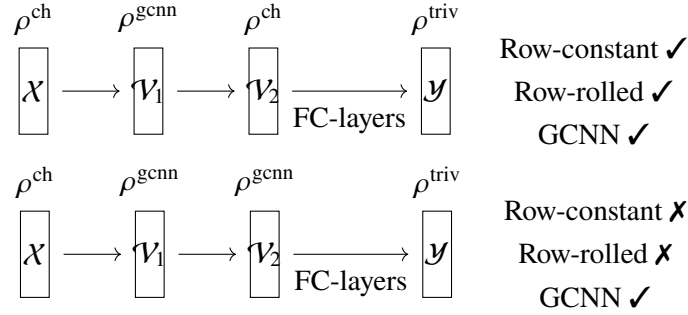
\begin{figure}[!h]
        \centering
        \begin{tikzpicture}[scale=0.38]
            \draw (0,0) rectangle (1,3);
            \node at (0.5,1.5) {$\gX$};
            \node at (0.5,4) {$\rho^{\mathrm{ch}}$};
            \draw (4,0) rectangle (5,3);
            \node at (4.5,1.5) {$\gV_1$};
            \node at (4.5,4) {$\rho^{\mathrm{gcnn}}$};
            \draw (8,0) rectangle (9,3);
            \node at (8.5,1.5) {$\gV_2$};
            \node at (8.5,4) {$\rho^{\mathrm{ch}}$};
            \node at (11.5,.5) {FC-layers};
            \draw (14,0) rectangle (15,3);
            \node at (14.5,1.5) {$\gY$};
            \node at (14.5,4) {$\rho^{\mathrm{triv}}$};
            \draw[->] (1.5,1.5) -- (3.5,1.5);
            \draw[->] (5.5,1.5) -- (7.5,1.5);
            \draw[->] (9.5,1.5) -- (13.5,1.5);

            \node at (20, 3) {Row-constant \cmark};
            \node at (20, 1.5) {Row-rolled \cmark};
            \node at (20, 0) {GCNN \cmark};
        \end{tikzpicture}
        
         \begin{tikzpicture}[scale=0.38]
            \draw (0,0) rectangle (1,3);
            \node at (0.5,1.5) {$\gX$};
            \node at (0.5,4) {$\rho^{\mathrm{ch}}$};
            \draw (4,0) rectangle (5,3);
            \node at (4.5,1.5) {$\gV_1$};
            \node at (4.5,4) {$\rho^{\mathrm{gcnn}}$};
            \draw (8,0) rectangle (9,3);
            \node at (8.5,1.5) {$\gV_2$};
            \node at (8.5,4) {$\rho^{\mathrm{gcnn}}$};
            \node at (11.5,.5) {FC-layers};
            \draw (14,0) rectangle (15,3);
            \node at (14.5,1.5) {$\gY$};
            \node at (14.5,4) {$\rho^{\mathrm{triv}}$};
            \draw[->] (1.5,1.5) -- (3.5,1.5);
            \draw[->] (5.5,1.5) -- (7.5,1.5);
            \draw[->] (9.5,1.5) -- (13.5,1.5);

            \node at (20, 3) {Row-constant \xmark};
            \node at (20, 1.5) {Row-rolled \xmark};
            \node at (20, 0) {GCNN \cmark};
        \end{tikzpicture}
\caption{Two ways of choosing the intermediate representations acting on the filter banks. All of the orbit expansion strategies are equivariant with respect to the top choice, but only the GCNN-strategy is equivariant with respect to the bottom choice. 
\label{fig: equivarization}}
        \end{figure}

    \begin{remark} Note that the middle construction is interesting in that it has an additional partial equivariance -- it is equivariant when $C_4$ acts \emph{only} through rolling on both input and output space, i.e. $(\rho^{\mathrm{roll}}(i)v)_k = v_{k-i}$. The corresponding calculation for the lifted representation becomes
    \begin{align*}
         (\rho^{}_{1}(i)v)_k &= \rho_{\gV_{2}}^{\mathrm{roll}}(i) ([\theta_\ell \circ \rho_{\gV_\ell}^{\mathrm{roll}}(i)^{-1}]v)_k = \sum_{j}\eta^{(k+i),j}*(\rho_{\gV_\ell}^{\mathrm{gcnn}}(i)^{-1}v)_j \\
    &=\sum_{j}\eta^{(k+i),j}*(v_{j-i}) = \sum_{j}\eta^{k+i,j+i}*v_j\,,
    \end{align*}
    and the equivariance condition becomes $\eta^{k+i,j+i}=\eta^{k,j}$ -- i.e. that the filters along each diagonal are equal. 
    \end{remark}

  \textbf{Experiment.} We evaluate all three variants with trigger epochs at 0, 20, and 100, using the same setup as \cref{sec:symmetrization-on-image-classification} (FashionMNIST, $N_{0} = 5\,000$, $C_4$ augmentation, SGD, 5 seeds). \cref{tab:orbit-expansion-ablation} reports metrics at the
best-accuracy checkpoint (the convention used throughout the paper).

The GCNN-expansion arrangement systematically outperformed the row-constant channel-wise expansion and row-rolled channel-wise expansion. This is not surprising given that they are the only ones that are equivariant when the $\rho^{\mathrm{gcnn}}$-representation is acting on all layers, which is the most expressive action in this setting. As we have seen, there are more filters that fulfill the weaker symmetry condition imposed by being equivariant from $\rho^{\mathrm{gcnn}}$ to $\rho^{\mathrm{ch}}$, intuitively leading to a weaker implicit bias. 

One can also additionally note that in the GCNN-mode, the outer rotation $R^{-i}$ and the inner shift index $(i-j)$ are oriented so that $i$ cancels in their composition, collapsing all rows to a single pattern. This leads to a lack of row-wise diversity, which seems to limit the effective degrees of freedom.

\section{Variational family closure ablation}
\label[app]{app:variational-family-closure-ablation}

Gaussian and Laplace families satisfy the conditions in \cref{thm:closure-under-push-forward} for linear group representations $\rho$, whereas a signed log-normal family with random, fixed sign patterns does not. The reason is that the sign assignments break the required affine structure of the sufficient statistics.

We design a controlled experiment to test whether the algebraic distinction has observable consequences during augmented variational training.

\paragraph{Setup.} We train Bayesian CNNs on FashionMNIST ($N_{0} = 20000$ training samples) with full $C_{4}$ orbit augmentation, using SGD for 100 epochs. Three variational families are compared: \textbf{(i)} Gaussian (mean-field, closed under $C_{4}$), \textbf{(ii)} Laplace (mean-field, closed under $C_{4}$), and \textbf{(iii)} signed log-normal with random fixed signs (not closed under $C_{4}$). All other hyper-parameters (architecture, learning rate, KL weights) are shared. Each configuration is run with 5 random seeds.

\paragraph{Metrics.} We measure the accuracy and symmetric KL divergence (Symm. KL) between the predictive distributions at $x$ and $gx$, evaluated on both train and test sets.

\paragraph{Results.} \cref{tab:closure-results} and \cref{fig:closure-trajectories} summarize the findings.

\begin{table}[h]
\centering
\small
\begin{tabular}{@{}l c c c c c@{}}
\toprule
 & & \multicolumn{2}{c}{Accuracy (\%)} & \multicolumn{2}{c}{Symm.\ KL ($\downarrow$)} \\
\cmidrule(lr){3-4} \cmidrule(lr){5-6}
Family & Closed? & Train & Test & Train & Test \\
\midrule
Gaussian       & \cmark & $90.79 \pm 0.17$ & $\mathbf{88.43 \pm 0.06}$ & $0.017 \pm 0.000$ & $\mathbf{0.023 \pm 0.001}$ \\
Laplace        & \cmark & $96.58 \pm 0.12$ & $87.74 \pm 0.30$ & $\mathbf{0.016 \pm 0.000}$ & $0.049 \pm 0.001$ \\
Signed LogNorm & \xmark & $90.31 \pm 0.13$ & $85.67 \pm 0.28$ & $0.036 \pm 0.002$ & $0.056 \pm 0.002$ \\
\bottomrule
\end{tabular}
\caption{Effect of variational-family closure on equivariance and accuracy under $C_4$ augmentation (epoch 100, mean $\pm$ std over 5 seeds). Symmetric KL divergence measures the equivariance defect (lower = more equivariant).}
\label{tab:closure-results}
\end{table}

\begin{figure}
    \centering
    \includegraphics[width=1\linewidth]{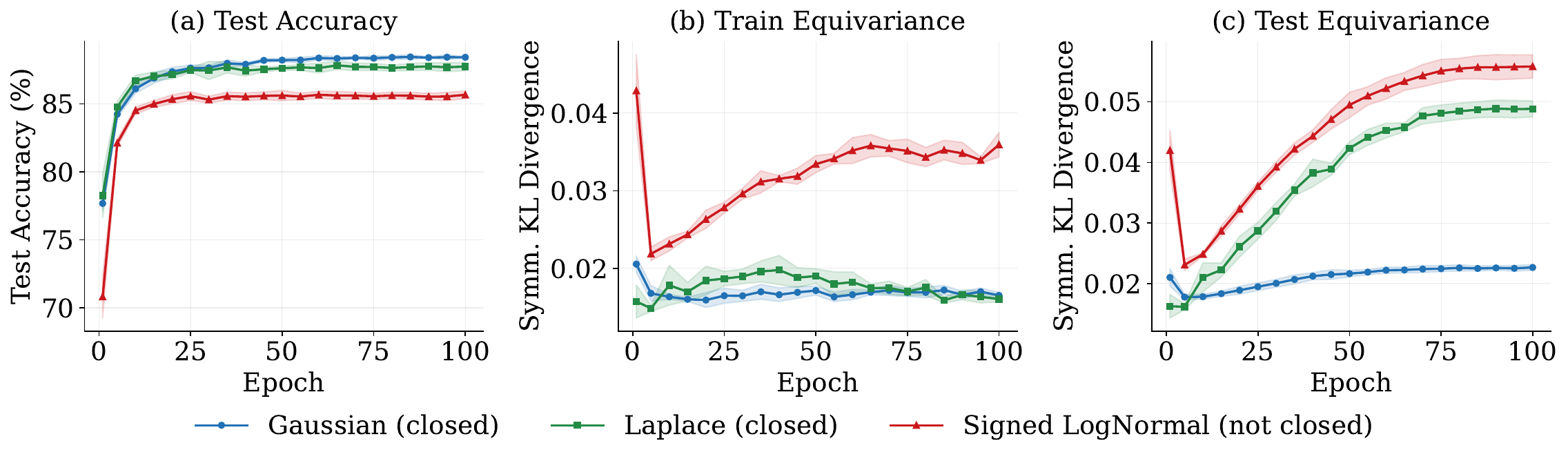}
    \caption{Trajectories for three variational families under $C_{4}$ augmentation (5 seeds, shaded region $= \pm$ std).
    \textbf{(a)} both closed families reach $\geq 87.7\%$ test accuracy; the non-closed signed log-normal plateaus at $\approx 85.7\%$.
    \textbf{(b)} Gaussian and Laplace converge to the same low train equivariance defect ($\approx 0.016$), confirming that closure enables the optimizer to find a near-equivariant posterior on training data, while the signed log-normal remains $> 2\times$ higher.
    \textbf{(c)} the test equivariance defect of all families grows during training as posteriors concentrate, but the ordering Gaussian $<$ Laplace $<$ log-normal is maintained throughout.}
    \label{fig:closure-trajectories}
\end{figure}

Both closed families converge to the same low train symmetric KL divergence ($\approx 0.016$), despite differing in distributional shape and training-set accuracy (Gaussian $90.79\%$, Laplace $96.58\%$). The signed log-normal, which is not closed, converges to a $2.2\times$ higher symmetric KL divergence ($0.036$), with non-overlapping seed ranges. This confirms the theorem's prediction: closure enables the optimizer to move the posterior toward a (near-)equivariant solution within the family, whereas a non-closed family faces a structural barrier.

We further find that equivariance acts as an implicit regularizer. Both closed families generalize better than the non-closed one: the Gaussian generalization gap is $2.4$ pp (train $90.79\%$ vs.\ test $88.43\%$), roughly half the $4.6$ pp gap of the signed log-normal (train $90.31\%$ vs.\ test $85.67\%$) at comparable training accuracy. This is consistent with equivariance reducing the effective degrees of freedom in the posterior.

While Laplace's train equivariance matches Gaussian, its test equivariance is noticeably worse ($0.049$ vs.\ $0.023$). We attribute this to the heavier tails of the Laplace distribution. The posterior satisfies the closure condition on training data, but the tail behaviour introduces additional prediction variance on unseen inputs, partially offsetting the equivariance benefit. This suggests that closure under group action is necessary but not sufficient for strong test-time equivariance; the posterior's concentration properties also play a role.

\section{Trigger timing and optimizer ablation}
\label[app]{app:trigger-timing-ablation}

\begin{table*}[t]
    \centering
    \small
    \setlength{\tabcolsep}{5pt}
    \begin{tabular}{@{}lll ccc@{}}
    \toprule
    \textbf{Method} & \textbf{Optimizer} & \textbf{Epoch}
      & \textbf{Acc.\,(\%)} $\uparrow$
      & \textbf{OSP} $\uparrow$
      & \textbf{Sym.\,KL} $\downarrow$ \\
    \midrule
    \multirow{2}{*}{Baseline} & AdamW & ---
      & $79.1{\scriptstyle\pm0.6}$ & $.923{\scriptstyle\pm.004}$ & $.0243{\scriptstyle\pm.0022}$ \\
       & SGD & ---
      & $79.0{\scriptstyle\pm0.6}$ & $.927{\scriptstyle\pm.003}$ & $.0217{\scriptstyle\pm.0028}$ \\
    \midrule
      \multirow{10}{*}{Geometric avg.} & \multirow{5}{*}{AdamW} & 0 & $77.7{\scriptstyle\pm1.0}$ & $.944{\scriptstyle\pm.002}$ & $.0118{\scriptstyle\pm.0009}$ \\
       &  & 20 & $77.9{\scriptstyle\pm0.4}$ & $.949{\scriptstyle\pm.005}$ & $.0096{\scriptstyle\pm.0011}$ \\
       &  & 50 & $78.2{\scriptstyle\pm0.5}$ & $.953{\scriptstyle\pm.008}$ & $.0090{\scriptstyle\pm.0033}$ \\
       &  & 100 & $77.3{\scriptstyle\pm0.9}$ & $.955{\scriptstyle\pm.008}$ & $.0076{\scriptstyle\pm.0025}$ \\
       &  & 150 & $76.8{\scriptstyle\pm0.5}$ & $.955{\scriptstyle\pm.009}$ & $.0074{\scriptstyle\pm.0026}$ \\
    \cmidrule(l){2-6}
       & \multirow{5}{*}{SGD} & 0 & $78.6{\scriptstyle\pm0.4}$ & $.950{\scriptstyle\pm.003}$ & $.0094{\scriptstyle\pm.0009}$ \\
       &  & 20 & $78.6{\scriptstyle\pm0.8}$ & $.952{\scriptstyle\pm.005}$ & $.0094{\scriptstyle\pm.0020}$ \\
       &  & 50 & $78.0{\scriptstyle\pm0.4}$ & $.947{\scriptstyle\pm.006}$ & $.0116{\scriptstyle\pm.0031}$ \\
       &  & 100 & $78.0{\scriptstyle\pm1.1}$ & $.946{\scriptstyle\pm.005}$ & $.0112{\scriptstyle\pm.0020}$ \\
       &  & 150 & $77.9{\scriptstyle\pm1.3}$ & $.944{\scriptstyle\pm.006}$ & $.0118{\scriptstyle\pm.0025}$ \\
    \midrule
      \multirow{10}{*}{Projection} & \multirow{5}{*}{AdamW} & 0 & $77.9{\scriptstyle\pm0.8}$ & $.922{\scriptstyle\pm.002}$ & $.0215{\scriptstyle\pm.0017}$ \\
       &  & 20 & $77.8{\scriptstyle\pm0.6}$ & $.924{\scriptstyle\pm.004}$ & $.0215{\scriptstyle\pm.0016}$ \\
       &  & 50 & $76.8{\scriptstyle\pm0.2}$ & $.926{\scriptstyle\pm.006}$ & $.0200{\scriptstyle\pm.0023}$ \\
       &  & 100 & $76.2{\scriptstyle\pm0.6}$ & $.928{\scriptstyle\pm.004}$ & $.0185{\scriptstyle\pm.0020}$ \\
       &  & 150 & $76.0{\scriptstyle\pm0.2}$ & $.928{\scriptstyle\pm.006}$ & $.0170{\scriptstyle\pm.0011}$ \\
    \cmidrule(l){2-6}
       & \multirow{5}{*}{SGD} & 0 & $78.4{\scriptstyle\pm0.8}$ & $.931{\scriptstyle\pm.004}$ & $.0174{\scriptstyle\pm.0026}$ \\
       &  & 20 & $78.5{\scriptstyle\pm0.2}$ & $.931{\scriptstyle\pm.004}$ & $.0180{\scriptstyle\pm.0026}$ \\
       &  & 50 & $77.8{\scriptstyle\pm0.9}$ & $.927{\scriptstyle\pm.004}$ & $.0194{\scriptstyle\pm.0020}$ \\
       &  & 100 & $77.2{\scriptstyle\pm0.9}$ & $.929{\scriptstyle\pm.012}$ & $.0182{\scriptstyle\pm.0046}$ \\
       &  & 150 & $76.3{\scriptstyle\pm1.3}$ & $.921{\scriptstyle\pm.006}$ & $.0208{\scriptstyle\pm.0023}$ \\
    \midrule
      \multirow{10}{*}{Orbit expansion} & \multirow{5}{*}{AdamW} & 0 & $77.8{\scriptstyle\pm1.0}$ & $.932{\scriptstyle\pm.004}$ & $.0170{\scriptstyle\pm.0025}$ \\
       &  & 20 & $78.0{\scriptstyle\pm0.6}$ & $.937{\scriptstyle\pm.010}$ & $.0145{\scriptstyle\pm.0041}$ \\
       &  & 50 & $77.9{\scriptstyle\pm0.2}$ & $.944{\scriptstyle\pm.004}$ & $.0114{\scriptstyle\pm.0013}$ \\
       &  & 100 & $77.7{\scriptstyle\pm0.4}$ & $.950{\scriptstyle\pm.004}$ & $.0090{\scriptstyle\pm.0015}$ \\
       &  & 150 & $77.5{\scriptstyle\pm0.5}$ & $.949{\scriptstyle\pm.006}$ & $.0085{\scriptstyle\pm.0015}$ \\
    \cmidrule(l){2-6}
       & \multirow{5}{*}{SGD} & 0 & $\mathbf{80.4}{\scriptstyle\pm0.4}$ & $.963{\scriptstyle\pm.003}$ & $.0057{\scriptstyle\pm.0004}$ \\
       &  & 20 & $80.4{\scriptstyle\pm0.6}$ & $.966{\scriptstyle\pm.004}$ & $.0048{\scriptstyle\pm.0009}$ \\
       &  & 50 & $80.2{\scriptstyle\pm0.5}$ & $.966{\scriptstyle\pm.003}$ & $.0050{\scriptstyle\pm.0008}$ \\
       &  & \cellcolor{gray!15}100 & \cellcolor{gray!15}$80.0{\scriptstyle\pm0.5}$ & \cellcolor{gray!15}$\mathbf{.968}{\scriptstyle\pm.003}$ & \cellcolor{gray!15}$\mathbf{.0042}{\scriptstyle\pm.0007}$ \\
       &  & 150 & $79.9{\scriptstyle\pm0.5}$ & $.967{\scriptstyle\pm.002}$ & $.0044{\scriptstyle\pm.0003}$ \\
    \bottomrule
    \end{tabular}
    \caption{\textbf{Trigger-timing ablation, full results.} Symmetrization mechanism $\times$ optimizer $\times$ trigger epoch on FashionMNIST ($N_0=5\,000$, $C_4$ augmentation, $10\,000$ test samples). For \emph{Geometric avg.} and \emph{Projection}, Epoch is the epoch at which the one-shot symmetrization is applied; for \emph{Orbit expansion} it is the number of stage-1 epochs before expansion (stage-2 runs for the remaining $500-\text{Epoch}$ epochs). All values are mean $\pm$ std over 5 seeds, taken at the best test-accuracy checkpoint. \textbf{Bold}: best across the whole table per column. \colorbox{gray!15}{Shaded}: best equivariance--accuracy trade-off.}
    \label{tab:trigger-timing-full}
\end{table*}

We empirically study the trigger-timing trade-off discussed in \cref{sec:symmetrization-of-the-variational-posterior}, and the role of the optimizer. The setup is identical to \cref{sec:symmetrization-on-image-classification} (FashionMNIST $N_{0} = 5000$, $C_{4}$ augmentation, 5 seeds; the orbit-expansion variant is the GCNN-expansion of \cref{app:expansion-along-orbit-filter-arrangement-ablation}), extended to the trigger epoch set $\{0, 20, 50, 100, 150\}$. After the trigger, all methods continue training without symmetry constraints up to epoch 500. We consider both SGD as well as AdamW.

We compare the equivariance properties of the best model (highest test accuracy) of each method in \cref{tab:trigger-timing-full}. There are several aspects of the results to be commented on.

\textbf{AdamW vs. SGD behavior.} First, the results from the main paper for the SGD runs are confirmed: an earlier trigger timing generally leads to a higher accuracy and equivariance (with the exception of the orbit expansion method, which however operates in a region of high equivariance). The AdamW runs behave differently: while their accuracies are still raised by the increased training time post re-initialization, the equivariance is generally hurt. This effect is better illustrated in Figure \ref{fig:adamWvsSGD}. This confirms the relevance of our theory, since Theorem \ref{thm:equivariance-of-vi}(ii) applies only to SGD. One should also note that although the AdamW-runs seem to drift more from the equivariant positions they are put in by our strategies, they are still made more equivariant than the baseline.

\textbf{Drift away from $H_G$.}  In Figure \ref{fig:gcnn-drift}, we plot the evolution of the symmetric KL-divergence for the example of orbit expansion. From this plot it is clear that the runs drift away from equivariance after the re-initalization. This effect persists across all runs and strategies, both for AdamW and SGD. Note that this is not in contradiction to our theory. Indeed, our results are technically only valid for gradient descent applied to the loss associated with the \emph{true predictive} maps. In the practical runs, we are using stochastic gradient descent applied to Monte Carlo approximations of the true predictive map. Importantly, the drift is a lot slower for SGD, further confirming the role of Theorem~\ref{thm:equivariance-of-vi}(ii). 

\textbf{Trigger timing.} The way the trigger timing influences the drift is subtle. In \cref{fig:proj-vs-avg-trajectory}, we plot the accuracy and equivariance metrics for the geometric averaging and projection metrics for two trigger timings. We see that both the immediately projected and the later projected runs drift -- but that the earlier trigged runs drift less. We hypothesize that when the projection is applied later in training, it will project from a point further from the typical point a random initialization yields. It is therefore likely that it is projected to a point far away from the optimal region, with large gradients as a result. This should intuitively make the discretization effects more severe, and the drift faster as a result. This is indeed what we observe in the figures: The orbit averaging methods affect the accuracy negatively to a point where it can't recover over the remaining epochs, and also drift away from equivariance quickly. We also speculate that this is one of the reasons of the performance difference of projection and geometric averaging --  it is plausible that the more non-local nature of the projection yield even more unlikely points post projection. We leave the task of providing a more complete examination of this phenomenon to future work.

\begin{figure}
    \centering
    \includegraphics[width=0.8\linewidth]{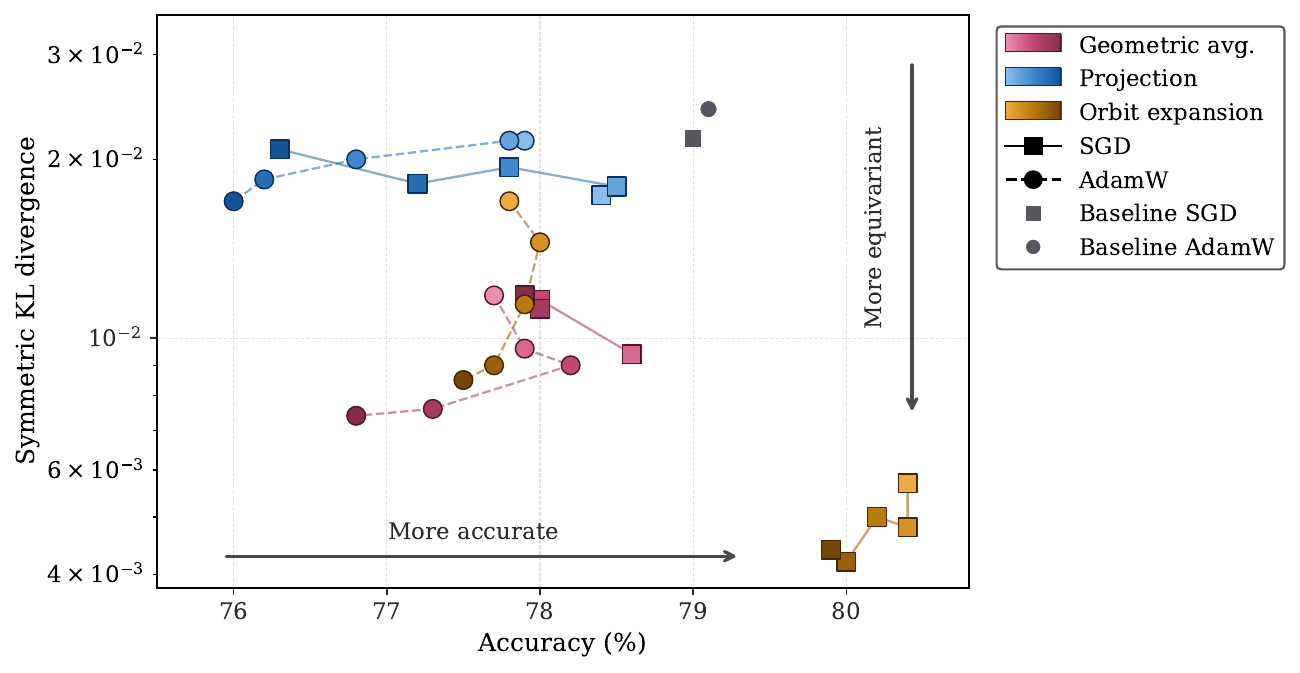}
    \caption{\textbf{Accuracy versus equivariance across methods, optimizers, and
    trigger timings} (FashionMNIST, $N_0 = 5000$, $C_4$; best-accuracy checkpoint,
    mean over 5 seeds). The $x$-axis is test accuracy and the $y$-axis is symmetric
    KL divergence on a log scale, so the best region is the \emph{lower right}
    (more accurate \emph{and} more equivariant). Hue encodes the method: rose for
    geometric averaging, blue for projection, amber for orbit expansion. Within
    each method the marker shade runs light$\rightarrow$dark for trigger epochs
    $\{0, 20, 50, 100, 150\}$ (legend). Marker shape and line style encode the
    optimizer: squares with solid lines for SGD, circles with dashed lines for
    AdamW. Grey markers are the no-symmetrization baselines (square: SGD, circle:
    AdamW). Orbit expansion under SGD sits in the lower-right corner, attaining the
    best accuracy and equivariance simultaneously; SGD matches or exceeds AdamW on
    accuracy throughout, and projection is the least equivariant of the three
    mechanisms.}
    \label{fig:adamWvsSGD}
\end{figure}

\begin{figure}[t]
    \centering
    \includegraphics[width=\linewidth]{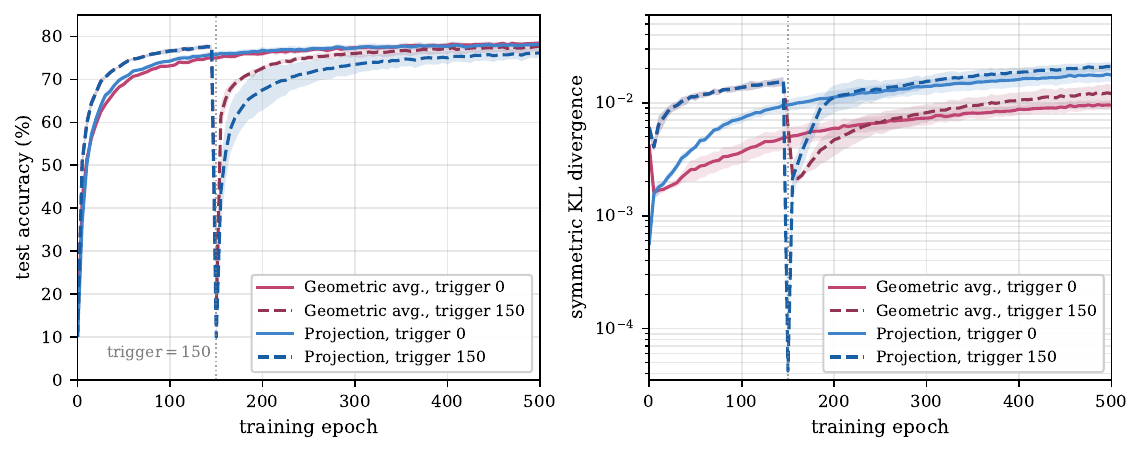}
    \caption{\textbf{Training trajectories under projection vs.\ geometric
    averaging with early vs.\ late triggers}
    (FashionMNIST, $N_0 = 5000$, $C_4$, SGD; single seed,
    representative of the 5-seed runs in \cref{tab:trigger-timing-full}).
    \textbf{Left:} test accuracy.  At trigger~150 (dotted line) both
    methods cause a brief collapse and recover, but geometric
    averaging returns almost back to its trigger-0 level
    while projection does not.
    \textbf{Right:} symmetric KL divergence on a log scale. The
    post-trigger drift out of $H_{G}$ has a similar shape and end
    value for both trigger settings within each method, consistent
    with the equivariance metric being controlled mainly by the
    drift dynamics rather than by trigger timing.}
    \label{fig:proj-vs-avg-trajectory}
\end{figure}

\begin{figure}[t]
    \centering
    \includegraphics[width=\linewidth]{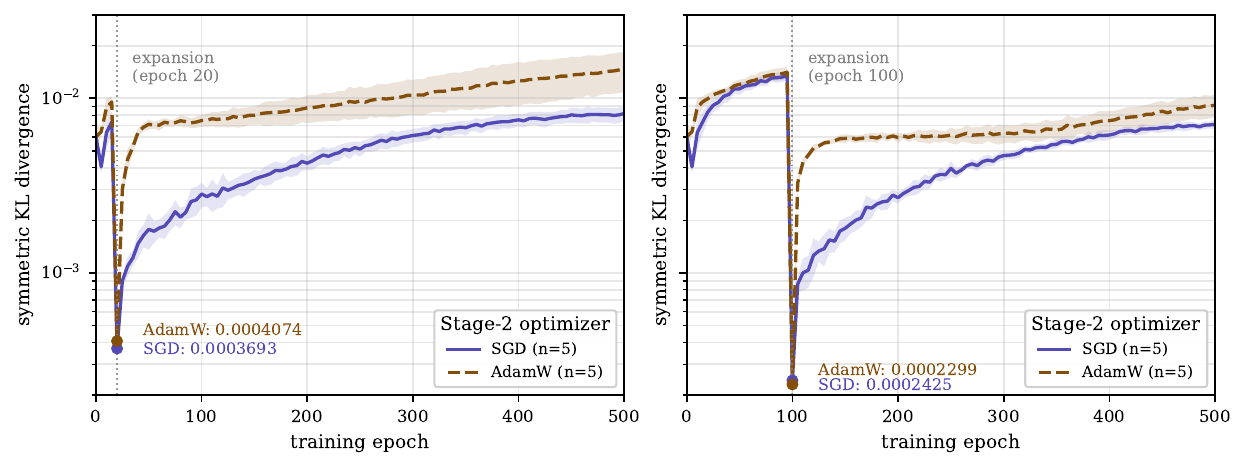}
    \caption{\textbf{Drift out of $H_{G}$ under SGD vs.\ AdamW for orbit
    expansion}, at two expansion times (FashionMNIST, $N_0 = 5000$, $C_4$;
    mean over 5 seeds, shaded bands $\pm 1$ s.d.).  \textbf{Left:} Stage-1
    $= 20$ epochs; \textbf{Right:} Stage-1 $= 100$ epochs.  Before the dotted
    line both optimizers train the same width-$1/|G|$ base network and their
    curves coincide.  At expansion the symmetric KL of both runs drops to the
    same near-zero floor.
   During Stage~2 the curves separate: SGD drifts up slowly
    AdamW faster and with larger seed variance.
    Curves trace the full trajectory to epoch~500;
    \cref{tab:trigger-timing-full} values are at the best-accuracy checkpoint,
    which for SGD precedes the late drift and is therefore lower.
    }
    \label{fig:gcnn-drift}
\end{figure}

 }


\end{document}